%% file: madlib_kdd.tex
\documentclass[sigconf]{acmart}

\makeatletter                   %1
\def\mdseries@tt{m}             %1
\makeatother                    %1
\usepackage{color}
\usepackage{hyperref}           %5
\hypersetup{
    colorlinks=true,
    linkcolor=blue,
    filecolor=red,      
    urlcolor=magenta,
    breaklinks=true,            %3
}
\usepackage{breakurl}           %3

\usepackage{booktabs} % For formal tables

\usepackage{qtree}
\usepackage[title]{appendix}
\usepackage{caption}
\usepackage{subcaption}
\usepackage{makecell}
\usepackage{array,multirow,graphicx}
\usepackage{float}

\usepackage{amsmath,amssymb,amsthm}
\usepackage[algoruled,vlined,nofillcomment]{algorithm2e}

\usepackage{hyperref}
\usepackage{cleveref}
\usepackage[nomessages]{fp}
\newcommand{\gloveScale}[1]{\FPeval{\epsilonValue}{round(#1 / 17.1677, 0)} \epsilonValue}

\usepackage[algoruled,vlined,nofillcomment]{algorithm2e}
\Crefname{figure}{Fig.}{Figs.}% {<type>}{<singular>}{<plural>}
\Crefname{section}{Sec.}{Secs.}% {<type>}{<singular>}{<plural>}
\Crefname{theorem}{Thm.}{Thms.}% {<type>}{<singular>}{<plural>}

\usepackage[textsize=footnotesize,color=green!40]{todonotes}

\newtheorem{theorem}{Theorem}

\newcommand{\cX}{\mathcal{X}}
\newcommand{\cY}{\mathcal{Y}}
\newcommand{\cW}{\mathcal{W}}
\newcommand{\R}{\mathbb{R}}
\renewcommand{\Pr}{\mathsf{Pr}}
\newcommand{\norm}[1]{\| #1 \|}
\newcommand{\argmin}{\mathop{argmin}}
\newcommand{\supp}{\mathrm{supp}}

\newcommand{\adhoc}{\emph{ad hoc}}
\newcommand{\apriori}{\emph{a priori}}
\newcommand{\defacto}{\emph{de facto}}
\newcommand{\etc}{\emph{etc.}}
\newcommand{\eg}{\emph{e.g.}}
\newcommand{\Eg}{\emph{E.g.}}
\newcommand{\ie}{\emph{i.e.}}
\newcommand{\vs}{\emph{vs.}}
\usepackage{acronym}
\acrodef{DP}{Differential Privacy}
\acrodef{LDP}{Local \ac{DP}}
\acrodef{mDP}{Metric \ac{DP}} % TD: I think MDP would actually look less jarring, but happy to stick with mDP
\acrodef{NLP}{Natural Language Processing}
\acrodef{ML}{Machine Learning}
\acrodef{AI}{Artificial Intelligence}
\acrodef{biLSTM}{Bidirectional LSTM} % TD: Are we happy with not spelling out LSTM?
\acrodef{GESD}{Geometric mean of Euclidean and Sigmoid Dot product}
\acrodef{MAP}{Mean Average Precision}
\acrodef{MRR}{Mean Reciprocal Rank}
\acrodef{PII}{\emph{personally identifiable information}}
\acrodef{AUC}{Area Under Curve}

\setcopyright{none}
\begin{document}
\sloppy                         %4
%https://medium.com/@rvprasad/typesetting-a-document-in-acm-style-for-arxiv-46967b306f2f

%\title[Madlib]{Madlib: Privacy and Semantic-preserving \\ Redaction for Sensitive Text}
\title{Privacy- and Utility-Preserving Textual Analysis via Calibrated Multivariate Perturbations}
%\titlenote{Produces the permission block, and copyright information}
%\subtitle{Madlib Subtitle}
%\subtitlenote{The full version of the author's guide is available as
%  \texttt{acmart.pdf} document}

%\author{Oluwaseyi Feyisetan, Borja Balle, Thomas Drake, Tom Diethe}
%\affiliation{Amazon}
%\email{{sey,pigem,draket,tdiethe}@amazon.com}
\author{Oluwaseyi Feyisetan}
\affiliation{Amazon}
\email{sey@amazon.com}

\author{Borja Balle}
\affiliation{Amazon}
\email{pigem@amazon.co.uk}

\author{Thomas Drake}
\affiliation{Amazon}
\email{draket@amazon.com}

\author{Tom Diethe}
\affiliation{Amazon}
\email{tdiethe@amazon.co.uk}

\begin{abstract}
Accurately learning from user data while providing quantifiable privacy guarantees provides an opportunity to build better ML models while maintaining user trust.
This paper presents a formal approach to carrying out privacy preserving text perturbation using the notion of $d_{\chi}$-privacy designed to achieve geo-indistinguishability in location data. Our approach applies carefully calibrated noise to vector representation of words in a high dimension space as defined by word embedding models. We present a privacy proof that satisfies $d_{\chi}$-privacy where the privacy parameter $\varepsilon$ provides guarantees with respect to a distance metric defined by the word embedding space. We demonstrate how $\varepsilon$ can be selected by analyzing plausible deniability statistics backed up by large scale analysis on $\textsc{GloVe}$ and $\textsc{fastText}$ embeddings. We conduct privacy audit experiments against $2$ baseline models and utility experiments on $3$ datasets to demonstrate the tradeoff between privacy and utility for varying values of $\varepsilon$ on different task types. 
Our results demonstrate practical utility (< 2\% utility loss for training binary classifiers) while providing better privacy guarantees than baseline models.
%Our results provide insights into carrying out practical privatization on text-based applications for a broad range of tasks. 
\end{abstract}

%
% The code below should be generated by the tool at
% http://dl.acm.org/ccs.cfm
% Please copy and paste the code instead of the example below.
%


%\keywords{ACM proceedings, \LaTeX, text tagging}

\maketitle

\input{madlib_kdd_body}

\input{experiments}

\bibliographystyle{ACM-Reference-Format}
\bibliography{madlib_kdd}

\clearpage
\newpage

\end{document}

%% file: madlib_kdd_body.tex
% !TEX root = madlib_kdd.tex
\section{Introduction}
Privacy-preserving data analysis is critical in the age of \ac{ML} and \ac{AI} where the availability of data can provide gains over tuned algorithms. However, the inability to provide sufficient privacy guarantees impedes this potential in certain domains such as with user generated queries. As a result, computation over sensitive data has been an important goal in recent years \cite{dinur2003revealing,gentry2009fully}. On the other hand,
% Appreciating the true implication of compromised data privacy usually comes after the fact, and at a high cost to data custodians charged with securing curated information. 
% \cite{barbaro2006face,narayanan2008robust,abowd2018us,dinur2003revealing,pandurangan2014taxis,tockar2014riding,venkatadri2018privacy}. 
private data that has been inappropriately revealed carries a high cost, both in terms of reputation damage and potentially fines, to data custodians charged with securing curated information. In this context, we distinguish between security and privacy breaches as follows: a security breach is unintended or unauthorized system usage, while a privacy breach is unintended or unauthorized data disclosure during intended system uses \cite{bambauer2013privacy}. Unintended disclosures, or accidental publications which lead to re-identification have been two common causes of recent privacy breaches \cite{barbaro2006face,narayanan2008robust,venkatadri2018privacy,abowd2018us,dinur2003revealing,pandurangan2014taxis,tockar2014riding}. While it is possible to define rules and design access policies to improve data security, understanding the full spectrum of what can constitute a potential privacy infraction can be hard to predict \apriori{}. As a result, solutions such as pattern matching, \adhoc{} filters and anonymization strategies are provably non-private. This is because such approaches cannot anticipate what side knowledge an attacker can use in conjunction with the released dataset. One definition that takes into account the limitations of existing approaches by preventing data reconstruction and protecting against any potential side knowledge is \acl{DP}.

\acf{DP} \cite{dwork2006calibrating}, which originated in the field of statistical databases, is one of the foremost standards for defining and dealing with privacy and disclosure prevention. At a high level, a randomized algorithm is differentially private if its output distribution is \textit{similar} when the algorithm runs on two neighboring input databases. The notion of similarity is controlled by a parameter $\varepsilon \geq 0$ that defines the strength of the privacy guarantee (with $\varepsilon = 0$ representing absolute privacy, and $\varepsilon = \infty$ representing null privacy). Even though \ac{DP} has been applied to domains such as geolocation \cite{andres2013geo}, social networks \cite{narayanan2009anonymizing} and deep learning \cite{abadi2016deep,shokri2015privacy}, less attention has been paid to adapting variants of \ac{DP} to the context of \ac{NLP} and the text domain \cite{Coavoux2018PrivacypreservingNR,Weggenmann_2018}.

% dfb - leakages = deliberate disclosure of confidential information.
We approach the research challenge of preventing leaks of private information in text data by building on the quantifiable privacy guarantees of \ac{DP}.
In addition to these formal privacy requirements, we consider two additional requirements informed by typical deployment scenarios.
First, the private mechanism must map text inputs to text outputs. This enables the mechanism to be deployed as a filter into existing text processing pipelines without additional changes to other components of the system.
Such a requirement imposes severe limitations on the set of existing mechanisms one can use, and in particular precludes us from leveraging hash-based private data structures commonly used to identify frequent words \cite{erlingsson2014rappor,thakurta2017learning,wang2017locally}.
The second requirement is that the mechanism must scale to large amounts of data and be able to deal with datasets that grow over time.
This prevents us from using private data synthesis methods such as the ones surveyed in \cite{bowen2016comparative} because they suffer from severe scalability issues even in moderate-dimensional settings, and in general cannot work with datasets that grow over time.
Together, these requirements push us towards solutions where each data record is processed independently, similar to the setting in \emph{\ac{LDP}} \cite{kasiviswanathan2011can}.
To avoid the curse of dimensionality of standard \ac{LDP} we instead adopt $d_{\chi}$-privacy \cite{andres2013geo,chatzikokolakis2013broadening,alvim2018local}, a relaxed variant of local \ac{DP} where privacy is defined in terms of the distinguishability level between inputs (see \Cref{sec:mdp} for details).

Our main contribution is a scalable mechanism for \emph{text analysis} satisfying $d_{\chi}$-privacy.
The mechanism operates on individual data records -- adopting the \emph{one user, one word} model as a baseline corollary to the \emph{one user, one bit} model in the \ac{DP} literature \cite{kasiviswanathan2011can}. It takes a private input word $x$, and returns a privatized version $\hat{x}$ where the word in the original record has been `perturbed'.
The perturbation is obtained by first using a pre-determined \emph{word embedding} model to map text into a high-dimensional vector space, adding noise to this vectorial representation, and the projecting back to obtain the perturbed word.
The formal privacy guarantees of this mechanism can be interpreted as a degree of plausible deniability \cite{bindschaedler2017plausible} conferred to the contents of $x$ from the point of view of an adversary observing the perturbed output. We explore this perspective in detail when discussing how to tune the privacy parameters of our mechanism.

The utility of the mechanism is proportional to how well the semantics of the input text are preserved in the output.
The main advantage of our mechanism in this context is to allow a higher degree of semantics preservation by leveraging the geometry provided by word embeddings when perturbing the data.
In this work we measure semantics preservation by analyzing the performance obtained by using the privatized data on downstream \ac{ML} tasks including binary sentiment analysis, multi-class classification, and question answering.
This same methodology is typically used in \ac{NLP} to evaluate unsupervised learning of word embeddings \cite{schnabel2015evaluation}.

Our contributions in this paper can be summarized as follows:
\begin{list}{{\arabic{enumi}.}}{\usecounter{enumi}
\setlength{\leftmargin}{11pt}
\setlength{\listparindent}{-1pt}}
\vspace*{-.7ex}
%\begin{itemize}
\item We provide a formal approach to carrying out intent preserving text perturbation backed up by formal privacy analysis (\Cref{sec:mechanism}). %To the best of our knowledge, this paper is the first to use $d_{\chi}$-privacy in the context of language data.
\item We provide a principled way to select the privacy parameter $\varepsilon$ for $d_{\chi}$-privacy on text data based on geometrical properties of word embeddings (\Cref{sec:ps-we}).
\item We conduct analysis on two embedding models, providing insights into words in the metric space (\Cref{sec:embeddings}). We also show how the vectors respond to perturbations, connecting the geometry of the embedding with statistics of the $d_{\chi}$-privacy mechanism.% that are relevant from a privacy point of view.
\item We apply our mechanism to different experimental tasks, at different values of $\varepsilon$, 
demonstrating the trade-off between privacy and utility (\Cref{sec:ml_experiments}). 
%\end{itemize}
\end{list}

\section{Privacy Preserving Mechanism}\label{sec:mechanism}
Consider a single word $x$ submitted by a user interacting with an information system. \Eg{} $x$ might represent a response to a survey request, or an elicitation of a fine grained sentiment.
In particular, $x$ will contain semantic information about the intent the user is trying to convey, but it also encodes an idiosyncratic representation of the user's word choices. Even though the word might not be explicitly \ac{PII} in the traditional sense of passwords and phone numbers, recent research has shown that the choice of words can serve as a \emph{fingerprint} \cite{bun2018fingerprinting} via which tracing attacks are launched \cite{songshmatikovkdd}.
Our goal is to produce $\hat{x}$, a version of $x$ that preserves the original intent while thwarting this attack vector.
We start the section by giving a high-level description of the rationale behind our mechanism and describing the threat model. Then we recall some fundamental concepts of $d_{\chi}$-privacy. Finally, we provide a detailed description of our mechanism, together with a formal statement of its privacy guarantees.

\subsection{Mechanism Overview}

We start by providing a high-level description of our mechanism. 
Our mechanism applies a $d_{\chi}$-privacy mechanism $\hat{x} = M(x)$ to obtain a replacement for the given word $x$. Such replacement is sampled from a carefully crafted probability distribution to ensure that $\hat{x}$ conveys a similar semantic to $x$ while at the same time hiding any information that might reveal the identity of the user who generated $x$. Intuitively, the randomness introduced by $d_{\chi}$-privacy provides plausible deniability \cite{bindschaedler2017plausible} with respect to the original content submitted by the user. However, it also permits a curator of the perturbed words to perform aggregate sentiment analysis, or to cluster survey results without significant loss of utility.

\subsection{Utility Requirements and Threat Model}

When designing our mechanism we consider a threat model where a \emph{trusted curator} collects a word from each user as $x^{(1)}, x^{(2)}, \ldots$ and wishes to make them available in clear form to an \emph{analyst} for use in some downstream tasks (such as clustering survey responses or building ML models). The data is collected in the `\emph{one user}, \emph{one word}' model, and we do not seek to extend theoretical protections to aggregate user data in this model.
Unfortunately, providing words in the clear presents the challenge of unwittingly giving the analyst access to information about the users interacting with the system. This could be either in the form of some shared side knowledge between the user and the analyst \cite{korolova2009releasing}, or through an ML attack to learn which users frequently use a given set of words \cite{shokri2017membership,songshmatikovkdd}.
Our working assumption is that the exact word is not necessary to effectively solve the downstream tasks of interest, although the general semantic meaning needs to be preserved to some extent; the experiments in \Cref{sec:ml_experiments} give several examples of this type of use case.
Thus, we aim to transform each submission by using randomization to provide plausible deniability over any potential identifiers.% that might be present in it.

%most of which could not be predicted \apriori{} especially in the presence of arbitrary side knowledge. This could lead to user re-identification, as was the case with the released AOL search logs \cite{pass2006picture} where the New York Times determined that user No. $4417749$ who searched for ``landscapers in Lilburn, GA'' was in fact Thelma Arnold \cite{barbaro2006face}.

\subsection{Privacy over Metric Spaces}\label{sec:mdp}
Over the last decade, \acf{DP} \cite{dwork2006calibrating} has emerged as a \defacto{} standard for privacy-preserving data analysis algorithms. 
%One reason for such success is the robustness of \ac{DP} against critical pitfalls exhibited by previous attempts to formalize privacy in the context of data analysis algorithms. 
Several variants of \ac{DP} have been proposed in the literature to address a variety of settings depending on whether, for example, privacy is defined with respect to aggregate statistics and \ac{ML} models (\emph{curator DP}) \cite{dwork2006calibrating}, or privacy is defined with respect to the data points contributed by each individual (\emph{local DP}) \cite{kasiviswanathan2011can}.

Since our application involves privatizing individual words submitted by each user, \ac{LDP} would be the ideal privacy model to consider. However, \ac{LDP} has a requirement that renders it impractical for our application: it requires that the given word $x$ has a non-negligible probability of being transformed into \emph{any} other word $\hat{x}$, no matter how unrelated $x$ and $\hat{x}$ are. Unfortunately, this constraint makes it virtually impossible to enforce that the \emph{semantics} of $x$ are approximately captured by the privatized word $\hat{x}$, since the space of words grows with the size of the language vocabulary, and the number of words semantically related to $x$ will have vanishingly small probability under \ac{LDP}.

To address this limitation we adopt $d_{\chi}$-privacy \cite{chatzikokolakis2013broadening,alvim2018local}, a relaxation of local \ac{DP} that originated in the context of location privacy to address precisely the limitation described above. In particular, $d_{\chi}$-privacy allows a mechanism to report a user's location in a privacy-preserving manner, while giving higher probability to locations which are close to the current location, and negligible probability to locations in a completely different part of the planet. $d_{\chi}$-privacy was originally developed as an abstraction of the model proposed in \cite{andres2013geo} to address the privacy-utility trade-off in location privacy. %To the best of our knowledge, this paper is the first to use $d_{\chi}$-privacy in the context of language data.

Formally, $d_{\chi}$-privacy is defined for mechanisms whose inputs come from a set $\cX$ equipped with a distance function $d : \cX \times \cX \to \R_+$ satisfying the axioms of a metric (\ie{} identity of indiscernibles, symmetry and triangle inequality). The definition of $d_{\chi}$-privacy depends on the particular distance function $d$ being used and it is parametrized by a privacy parameter $\varepsilon > 0$.
%\tom{$\geq 0$ ?}
%\borja{Without a $\delta$ the case $\varepsilon = 0$ is not interesting. And in our case setting $\varepsilon = 0$ gives a distribution that cannot be normalized.}
We say that a randomized mechanism $M : \cX \to \cY$ satisfies $\varepsilon d_{\chi}$-privacy if for any $x, x' \in \cX$ the distributions over outputs of $M(x)$ and $M(x')$ satisfy the following bound: for all $y \in \cY$ we have
\begin{align}\label{eqn:mdp}
\frac{\Pr[M(x) = y]}{\Pr[M(x') = y]} \leq e^{\varepsilon d(x,x')} \enspace.
\end{align}
%Note that for an uncountable output space $\cY$ the above probabilities have to be interpreted as densities; here we follow the convention used in \ac{DP} to not make this distinction explicit in our notation.
We note that $d_{\chi}$-privacy exhibits the same desirable properties of \ac{DP} (\eg{} composition, post-processing, robustness against side knowledge, \etc{}), but we won't use these properties explicitly in our analysis; we refer the reader to \cite{chatzikokolakis2013broadening} for further details.

The type of probabilistic guarantee described by \eqref{eqn:mdp} is characteristic of \ac{DP}: it says that the log-likelihood ratio of observing any particular output $y$ given two possible inputs $x$ and $x'$ is bounded by $\varepsilon d(x,x')$. The key difference between $d_{\chi}$-privacy and local \ac{DP} is that the latter corresponds to a particular instance of the former when the distance function is given by $d(x,x') = 1$ for every $x \neq x'$. Unfortunately, this Hamming metric does not provide a way to classify some pairs of points in $\cX$ as being closer than others. This indicates that local \ac{DP} implies a strong notion of indistinguishability of the input, thus providing very strong privacy by ``remembering almost nothing'' about the input. Fortunately, $d_{\chi}$-privacy is less restrictive and allows the indistinguishability of the output distributions to be scaled by the distance between the respective inputs. In particular, the further away a pair of inputs are, the more distinguishable the output distributions can be, thus allowing these distributions to remember more about their inputs than under the strictly stronger definition of local \ac{DP}.
An inconvenience of $d_{\chi}$-privacy is that the meaning of the privacy parameter $\varepsilon$ changes if one considers different metrics, and is in general incomparable with the $\varepsilon$ parameter used in standard (local) \ac{DP} (which can lead to seemingly larger privacy budget values as the dimensionality of the metric space increases). As a result, this paper makes no claim to provide privacy guarantees in the traditional sense of classical \ac{DP}. Thus, in order to understand the privacy consequences of a given $\varepsilon$ in $d_{\chi}$-privacy one needs to understand the structure of the underlying metric $d$. For now we assume $\varepsilon$ is a parameter given to the mechanism; we will return to this point in \Cref{sec:ps-we} where we analyze the meaning of this parameter for metrics on words derived from embeddings. All the metrics described in this work are Euclidean. For discussions on $d_{\chi}$-privacy over other metrics (such as Manhattan and Chebyshev, see \cite{chatzikokolakis2013broadening})

%\subsection{Properties of \ac{mDP}} \ac{mDP} enjoys the same important theoretical properties as other variants of \ac{DP}.
%Here we recall the properties that will be relevant for our analysis.
%\begin{itemize}
%\item Composition properties.
%\item Post-processing.
%\item Exponential mechanism.
%\end{itemize}

\subsection{Method Details}
\label{mechanism_details}

We now describe the proposed $d_{\chi}$-privacy mechanism.
The full mechanism takes as input a string $x$ containing $|x|$ words and outputs a string $\hat{x}$ of the same length.
To privatize $x$ we use a $d_{\chi}$-privacy mechanism $M : \cX \to \cX$, where $\cX = \cW^{\ell}$ is the space of all strings of length $\ell$ with words in a dictionary $\cW$. The metric between strings that we consider here is derived from a \emph{word embedding model} $\phi : \cW \to \R^n$ as follows: given $x, x' \in \cW^{\ell}$ for some $\ell \geq 1$, we let $d(x,x') = \sum_{i = 1}^{\ell} \norm{\phi(w_i) - \phi(w_i')}$,
%\begin{align*}
%d(x,x') = \sum_{i = 1}^{\ell} \norm{\phi(w_i) - \phi(w_i')} \enspace,
%\end{align*}
where $w_i$ (resp.\ $w_i'$) denotes the $i$th word of $x$ (resp.\ $x'$),  and $\norm{\cdot}$ denotes the Euclidean norm on $\R^n$. Note that $d$ satisfies all the axioms of a metric as long as the word embedding $\phi$ is injective. We also assume the word embedding is independent of the data to be privatized; \eg{} we could take an available word embedding like GloVe \cite{pennington2014glove} or train a new word embedding on an available dataset.
%\tom{should we say what happens if $x \notin \cW^\ell$ - I assume we'd have to redact these completely?}
%\borja{we could assume $\cW$ contains an 'out-of-vocabulary' symbol, but I'd rather skip this detail in the math, it's more a systems issue in my view}
Our mechanism $M$ works by computing the embedding $\phi(w)$ of each word $w \in x$, adding some properly calibrated random noise $N$ to obtain a perturbed embedding $\hat{\phi} = \phi(w) + N$, and then replacing the word $w$ with the word $\hat{w}$ whose embedding is closest to $\hat{\phi}$. The noise $N$ is sampled from an $n$-dimensional distribution with density $p_N(z) \propto \exp(- \varepsilon \norm{z})$, where $\varepsilon$ is the privacy parameter of the mechanism. The following pseudo-code provides implementation details for our mechanism. 

\begin{algorithm}[h]
\small
\DontPrintSemicolon
\SetKw{KwRelease}{release}
\caption{Privacy Preserving Mechanism}\label{alg:madlib}
\KwIn{string $x = w_1 w_2 \cdots w_{\ell}$, privacy parameter $\varepsilon > 0$}
\For{$i \in \{1, \ldots, \ell\}$}{
Compute embedding $\phi_i = \phi(w_i)$\;
Perturb embedding to obtain $\hat{\phi}_i = \phi_i + N$ with noise density $p_N(z) \propto \exp(- \varepsilon \norm{z})$\;
Obtain perturbed word $\hat{w}_i = \argmin_{u \in \cW} \norm{\phi(u) - \hat{\phi}_i}$\;
Insert $\hat{w}_i$ in $i$th position of $\hat{x}$\;
}
\KwRelease{$\hat{x}$}
\end{algorithm}

See \Cref{sec:how_sample} for details on how to sample noise from the multivariate distribution $p_N$ for different values of $\varepsilon$.

\subsection{Privacy Proof}
The following result states that our mechanism $M$ satisfies $\varepsilon d_{\chi}$-privacy with respect to the metric $d$ defined above. 

\begin{theorem}\label{thm:M-mdp}
For any $\ell \geq 1$ and any $\varepsilon > 0$, the mechanism $M : \cW^{\ell} \to \cW^{\ell}$ satisfies $\varepsilon d_{\chi}$-privacy with respect to $d$.
\end{theorem}

\begin{proof}
The intuition behind the proof is to observe that $M$ can be viewed as a combination of the generic exponential mechanism construction for the metric $d$ together with a post-processing strategy that does not affect the privacy guarantee of the exponential mechanism. However, we chose not to formalize our proof in those terms; instead we provide a self-contained argument leading to a more direct proof without relying on properties of $d_{\chi}$-privacy established elsewhere.

To start the proof, we first consider the case $\ell = 1$ so that $x = w \in \cW$ and $x' = w' \in \cW$ are two inputs of length one. For any possible output word $\hat{w} \in \cW$ we define a set $C_{\hat{w}} \subset \R^n$ containing all the feature vectors which are closer to the embedding $\phi(\hat{w})$ than to the embedding of any other word. Formally, we have
\begin{align*}
C_{\hat{w}} = \left\{ z \in \R^n \;:\; \norm{z - \phi(\hat{w})} < \min_{u \in \cW \setminus \{\hat{w}\}} \norm{z - \phi(u)} \right\} \enspace.
\end{align*}
The set $C_{\hat{w}}$ is introduced because it is directly related to the probability that the mechanism $M$ on input $x = w$ produces $\hat{w}$ as output. Indeed, by the description of $M$ we see that we get $M(w) = \hat{w}$ if and only if the perturbed feature vector $\hat{\phi} = \phi(w) + N$ is closer to $\phi(\hat{w})$ than to the embedding of any other word in $\cW$. In particular, letting $p_{\phi(w) + N}(z)$ denote the density of the random variable $\phi(w) + N$, we can write the probability of this event as follows:
\begin{align*} % TD: I moved the alignment tab below to stop it going into the margin
\Pr[M(w) &= \hat{w}]
=
\Pr[ \phi(w) + N \in C_{\hat{w}} ]
=
\int_{C_{\hat{w}}} p_{\phi(w) + N}(z) dz \\
&=
\int_{C_{\hat{w}}} p_{N}(z - \phi(w)) dz
\propto
\int_{C_{\hat{w}}} \exp(- \varepsilon \norm{z - \phi(w)}) dz \enspace,
\end{align*}
where we used that $\phi(w) + N$ has exactly the same distribution of $N$ but with a different mean.
Now we note that the triangle inequality for the norm $\norm{\cdot}$ implies that for any $z \in \R^n$ we have the following inequality:
\begin{align*}
\exp(- \varepsilon \norm{z - \phi(w)})
&=
\frac{\exp(- \varepsilon \norm{z - \phi(w)})}{\exp(- \varepsilon \norm{z - \phi(w')})} \exp(- \varepsilon \norm{z - \phi(w')}) \\
&=
\exp(\varepsilon (\norm{z - \phi(w')} - \norm{z - \phi(w)})) \\ 
& \;\; \times \;
\exp(- \varepsilon \norm{z - \phi(w')}) \\
&\leq
\exp(\varepsilon \norm{\phi(w) - \phi(w')} ) \exp(- \varepsilon \norm{z - \phi(w')}) \\
&=
\exp(\varepsilon d(w,w') ) \exp(- \varepsilon \norm{z - \phi(w')}) \enspace.
\end{align*}
Combining the last two derivations and observing the the normalization constants in $p_N(z)$ and $p_{\phi(w) + N}(z)$ are the same, we obtain
\begin{align*}
\frac{\Pr[M(w) = \hat{w}]}{\Pr[M(w') = \hat{w}]}
&=
\frac{\int_{C_{\hat{w}}} \exp(- \varepsilon \norm{z - \phi(w)}) dz}{\int_{C_{\hat{w}}} \exp(- \varepsilon \norm{z - \phi(w')}) dz}
\leq
\exp(\varepsilon d(w,w') ) \enspace.
\end{align*}
Thus, for $\ell = 1$ the mechanism $M$ is $\varepsilon d_{\chi}$-privacy preserving.

Now we consider the general case $\ell > 1$. We claim that because the mechanism treats each word in $x = w_1 \cdots w_{\ell}$ independently, the result follows directly from the analysis for the case $\ell = 1$. To see this, we note the following decomposition allows us to write the output distribution of the mechanism on strings of length $\ell > 1$ in terms of the output distributions of the mechanism on strings of length one: for $x, \hat{x} \in \cW^\ell$ we have
\begin{align*}
\Pr[M(x) = \hat{x}] 
&= 
%\Pr[M(w_1 \cdots w_{\ell}) = \hat{w}_1 \cdots \hat{w}_{\ell}] \\
%&= 
\prod_{i = 1}^\ell \Pr[M(w_i) = \hat{w}_i] \enspace.
\end{align*}
Therefore, using that $M$ is $d_{\chi}$-privacy preserving with respect to $d$ on strings of length one, we have that for any pair of inputs $x, x' \in \cW^{\ell}$ and any output $\hat{x} \in \cW^{\ell}$ the following is satisfied:
\begin{align*}
\frac{\Pr[M(x) = \hat{x}]}{\Pr[M(x') = \hat{x}]}
&=
\prod_{i = 1}^{\ell} \left(
\frac{Pr[M(w_i) = \hat{w}_i]}{Pr[M(w_i') = \hat{w}_i]}
\right) \\
&
\leq
\prod_{i = 1}^{\ell} \exp(\varepsilon d(w_i,w_i'))
=
\exp(\varepsilon d(x,x')) \enspace,
\end{align*}
where we used that the definition of $d$ is equivalent to $d(x,x') = \sum_{i = 1}^\ell d(w_i,w_i')$. The result follows.
\end{proof}

\subsection{Sampling from the Noise Distribution}
\label{sec:how_sample}
To sample from $p_N$, first, we sample a vector-valued random variable $\mathbf{v} = [v_1 \ldots v_n]$ from the multivariate normal distribution:
\begin{align*}
p(x; \mu, \Sigma) = \frac{1}{(2 \pi)^{n/2}|\Sigma|^{1/2}} \exp \Big( -\frac{1}{2}(x - \mu)^T \Sigma^{-1}(x - \mu)  \Big) .
\end{align*}
where $n$ is the dimensionality of the word embedding, the mean $\mu$ is centered at the origin and the covariance matrix $\Sigma$ is the identity matrix. The vector $\mathbf{v}$ is then normalized to constrain it in the unit ball. Next, we sample a magnitude $l$ from the Gamma distribution  
\begin{align*}
p(x; n, \theta) = \frac{x^{n - 1} e^{-x/\theta}}{\Gamma(n) \theta^n} .
\end{align*}
where $\theta = 1/\varepsilon$ and $n$ is the embedding dimensionality. A sample noisy vector at the privacy parameter $\varepsilon$ is therefore output as $l \mathbf{v}$.
More details on the approach can be found in \cite[Appendix E]{wu2017bolt}.

\section{Statistics for Privacy Calibration}\label{sec:ps-we}

In this section we present a methodology for calibrating the $\varepsilon$ parameter of our $d_{\chi}$-privacy mechanism $M$ based on the geometric structure of the word embedding $\phi$ used to define the metric $d$. Our strategy boils down to identifying a small number of statistics associated with the output distributions of $M$, and finding a range of parameters $\varepsilon$ where these statistics behave as one would expect from a mechanism providing a prescribed level of plausible deniability. We recall that the main reason this is necessary, and why the usual rules of thumb for calibrating $\varepsilon$ in traditional (\ie{} hamming distance based) \ac{DP} cannot be applied here, is because the meaning of $\varepsilon$ in $d_{\chi}$-privacy depends on the particular metric being used and is not transferable across metrics.
We start by making some qualitative observations about how $\varepsilon$ affects the behavior of mechanism $M$. For the sake of simplicity we focus the discussion on the case where $x$ is a single word $x = w$, but all our observations can be directly generalized to the case $|x| > 1$.
%We note these observations are essentially heuristic, although it is not hard to turn them into precise mathematical statements.

\subsection{Qualitative Observations}

The first observation is about the behavior at extreme values of $\varepsilon$. As $\varepsilon \to 0$ we have $M(w)$ converging to the a fixed distribution over $\cW$ independent of $w$. This distribution will not be uniform across $\cW$ since the probability $\lim_{\varepsilon \to 0} \Pr[M(w) = \hat{w}]$ will depend on the relative size of the event $C_{\hat{w}}$ defined in the proof of \Cref{thm:M-mdp}. However, since this distribution is the same regardless of the input $w$, we see that $\varepsilon \to 0$ provides absolute privacy as the output produced by the mechanism becomes independent of the input word. Such a mechanism will not provide preserve semantics as the output is essentially random.
In contrast, the regime $\varepsilon \to \infty$ will yield a mechanism satisfying $M(w) = w$ for all inputs, thus providing null privacy, but fully preserving the semantics.  %of the input.
As expected, by tuning the privacy parameter $\varepsilon$ we can trade-off privacy \vs{} utility.
Utility for our mechanism comes in the form of some semantic-preserving properties; we will measure the effect of $\varepsilon$ on the utility when we use the outputs in the context of a \ac{ML} pipeline in \Cref{sec:ml_experiments}.
Here we focus on exploring the effect of $\varepsilon$ on the privacy provided by the mechanism, so as to characterize the minimal values of the parameter that yield acceptable privacy guarantees.

Our next observation is that for any finite $\varepsilon$, the distribution of $M(w)$ has full support on $\cW$. In other words, for any possible output $\hat{w} \in \cW$ we have a non-zero probability that $M(w) = \hat{w}$. However, we know from our discussion above that for $\hat{w} \neq w$ these probabilities vanish as $\varepsilon \to \infty$. A more precise statement can be made if one tries to compare the rate at which the probabilities $\Pr[M(w) = \hat{w}]$ for different outputs $\hat{w}$. In particular, given two outputs with $d(w, \hat{w}) \ll d(w, \hat{w}')$, by the definition of $M$ we will have $\Pr[M(w) = \hat{w}] \gg \Pr[M(w) = \hat{w}']$ for any fixed $\varepsilon$. Thus, taking the preceding observation and letting $\varepsilon$ grow, one obtains that $\Pr[M(w) = \hat{w}]$ goes to zero much faster for outputs $\hat{w}$ far from $w$ than for outputs close to it.
We can see from this argument that, essentially, as $\varepsilon$ grows, the distribution of $M(w)$ concentrates around $w$ and the words close to $w$. This is good from a utility point of view -- words close to $w$ with respect to the metric $d$ will have similar meanings by the construction of the embeddings -- but too much concentration degrades the privacy guarantee since it increases the probability $\Pr[M(w) = w]$ and makes the effective support of the distribution of $M(w)$ too small to provide plausible deniability.

\subsection{Plausible Deniability Statistics}
\label{sec:plausible_deny}

Inspired by the discussion above, we define two statistics to measure the amount of plausible deniability provided by a choice of the privacy parameter $\varepsilon$. 
Roughly speaking, in the context of text redaction applications, plausible deniability measures the likelihood of making correct inferences about the input given a sample from the privatization mechanism.
In this sense, plausible deniability can be achieved by making sure the original word has low probability of being released unperturbed, and additionally making sure that the words that are frequently sampled given some input word induce enough variation on the sample to hide which what the input word was.
A key difference between \ac{LDP} and $d_{\chi}$-privacy is that the former provides a stronger form of plausible deniability by insisting that almost every outcome is possible when a word is perturbed, while the later only requires that we give enough probability mass to words close to the original one to ensure that the output does not reveal what the original word was, although it still releases information about the neighborhood where the original word was.

More formally, the statistics we look at are the probability $N_w = \Pr[M(w) = w]$ of not modifying the input word $w$, and the (effective) support of the output distribution $S_w$ (\ie{} number of possible output words) for an input $w$. In particular, given a small probability parameter $\eta > 0$, we define $S_w$ as the size of the smallest set of words that accumulates probability at least $1 - \eta$ on input $w$:
\begin{align*}
S_w = \min | \{ S \subseteq \cX \; : \; \Pr[M(w) \notin S] \leq \eta \} | \enspace.
\end{align*}
Intuitively, a setting of $\varepsilon$ providing plausible deniability should have $N_w$ small and $S_w$ large for (almost) all words in $w \in \cW$.

These statistics can also be related to the two extremes of the R{\'e}nyi entropy \cite{renyi1961measures}, thus providing an additional information-theoretic justification for the settings of $\varepsilon$ that provide plausible deniability in terms of large entropy. Recall that for a distribution $p$ over $\cW$ with $p_w = \Pr_{W \sim p}[W = w]$, the R{\'e}nyi entropy of order $\alpha \geq 0$ is% given by
\begin{align*}
H_{\alpha}(p) = \frac{1}{1 - \alpha} \log \left( \sum_{w \in \cW} p_w^{\alpha} \right) \enspace.
\end{align*}
By taking the extremes $\alpha = 0$ and $\alpha = \infty$ one obtains the so-called min-entropy $H_0(p) = \log | \supp(p) |$ and max-entropy $H_{\infty}(p) = \log(1 / \max_w p_w)$, where $\supp(p) = \{ w : p_w > 0 \}$ denotes the support of $p$.
This implies that we can see the quantities $S_w$ and $N_w$ as proxies for the two extreme R{\'e}nyi entropies through the approximate identities $H_0(M(w)) \approx \log S_w$ and $H_{\infty}(M(w)) \approx \log 1 / N_w$, where the last approximation relies on the fact that (at least for small enough $\varepsilon$), $w$ should be the most likely word under the distribution of $M(w)$. Making these two quantities large amounts to increasing the entropy of the distribution. In practice, we prefer to work with the statistics $S_w$ and $N_w$ than with the extreme R{\'e}nyi entropies since the former are easier to estimate through simulation.

%% file: experiments.tex
\section{Analysis of Word Embeddings}\label{sec:embeddings}

A word embedding $\phi : \cW \to \R^n$ maps each word in some vocabulary to a vector of real numbers.
An approach for selecting the model parameters is to posit a conditional probability $p(o | w)$ of observing a word $o$ given a nearby word or context $w$ by taking the soft-max over all contexts in the vocabulary as: %$p(o | w) = \exp\left({\phi(o)}^\top \phi(w)\right) / \sum_{w' \in \cW} \exp\left({\phi(w')}^\top \phi(w)\right)$ 
\cite{goldberg2014word2vec}.
\begin{align*}
p(o | w) = \frac{\exp\left({\phi(o)}^\top \phi(w)\right)}{\sum_{w' \in \cW} \exp\left({\phi(w')}^\top \phi(w)\right)}
%p(o | w) = \exp\left({\phi(o)}^\top \phi(w)\right) / \textstyle\sum_{w' \in \cW} \exp\left({\phi(w')}^\top \phi(w)\right)
\end{align*}
Such models are usually trained using a skip-gram objective \cite{mikolov2013distributed} to maximize the average log probability of words $w$ given the surrounding words as a context window of size $m$ scans through a large corpus of words $w_1, \ldots, w_T$: $ \frac{1}{T} \sum_{t = 1}^{T} \sum_{-m \leq j \leq m, j \neq 0} \log p (w_{t+j} | w_t)$.

% removed_for_kdd
\begin{comment}
\begin{align*}
\frac{1}{T} \sum_{t = 1}^{T} \sum_{-m \leq j \leq m, j \neq 0} \log p (w_{t+j} | w_t)\enspace,
\end{align*}
\end{comment}
% removed_for_kdd

% removed_for_kdd
\begin{comment}
\begin{align*}
p(o | w) = \frac{\exp\Big({v_o}^\top v_w\Big)}{\sum_{c \in C} \exp\Big({v_c}^\top v_w\Big)}
\end{align*}
\end{comment}
% removed_for_kdd

The geometry of the resulting embedding model has a direct impact on defining the output distribution of our redaction mechanism.
To get an intuition for the structure of these metric spaces -- \ie{}, how words cluster together and the distances between words and their neighbors -- we ran several analytical experiments on two widely available word embedding models: \textsc{GloVe} \cite{pennington2014glove} and \textsc{fastText} \cite{bojanowski2016enriching}.
We selected $319,000$ words that were present in both the \textsc{GloVe} and \textsc{fastText} embeddings. 
%We normalized each vector by the Euclidean norm of the largest vector (computed separately for GloVe and fastText). 
Though we present findings only from the common $319,000$ words in the embedding vocabularies, we carried out experiments over the entire vector space (\ie{}, $400,000$ for \textsc{GloVe} and $2,519,370$ for \textsc{fastText}).

Our experiments provide: (i) insights into the distance $d(x,x')$ that controls the privacy guarantees of our mechanism for different embedding models detailed below; and (ii) empirical evaluation of the plausible deniability statistics $S_w$ and $N_w$ described in \Cref{sec:word_distribution} for the mechanisms obtained using different embeddings.

We analyzed the distance between each of the $319,000$ words and its $k$ closest neighbors. The $k$ values were $1, 5, 10, 20, 50, 100, 200, 500$, and $1000$. We computed the Euclidean distance between each word vector and its $k$ neighbors. We then computed $5$th, $20$th, $50$th, $80$th, and $95$th percentile of the distances for each of the $k$ values. The line chart in \Cref{fig:emb_summary} summarizes the results across the percentiles values by presenting a logarithmic view of the increasing $k$ values. %\Cref{fig:emb_summary} also presents a histogram (for $k = 10$) to reflect how the distances are distributed between words and their neighbors in the two embedding models.

\begin{figure}[h]
\begin{subfigure}[t]{\linewidth}
    \includegraphics[width=.47\linewidth]{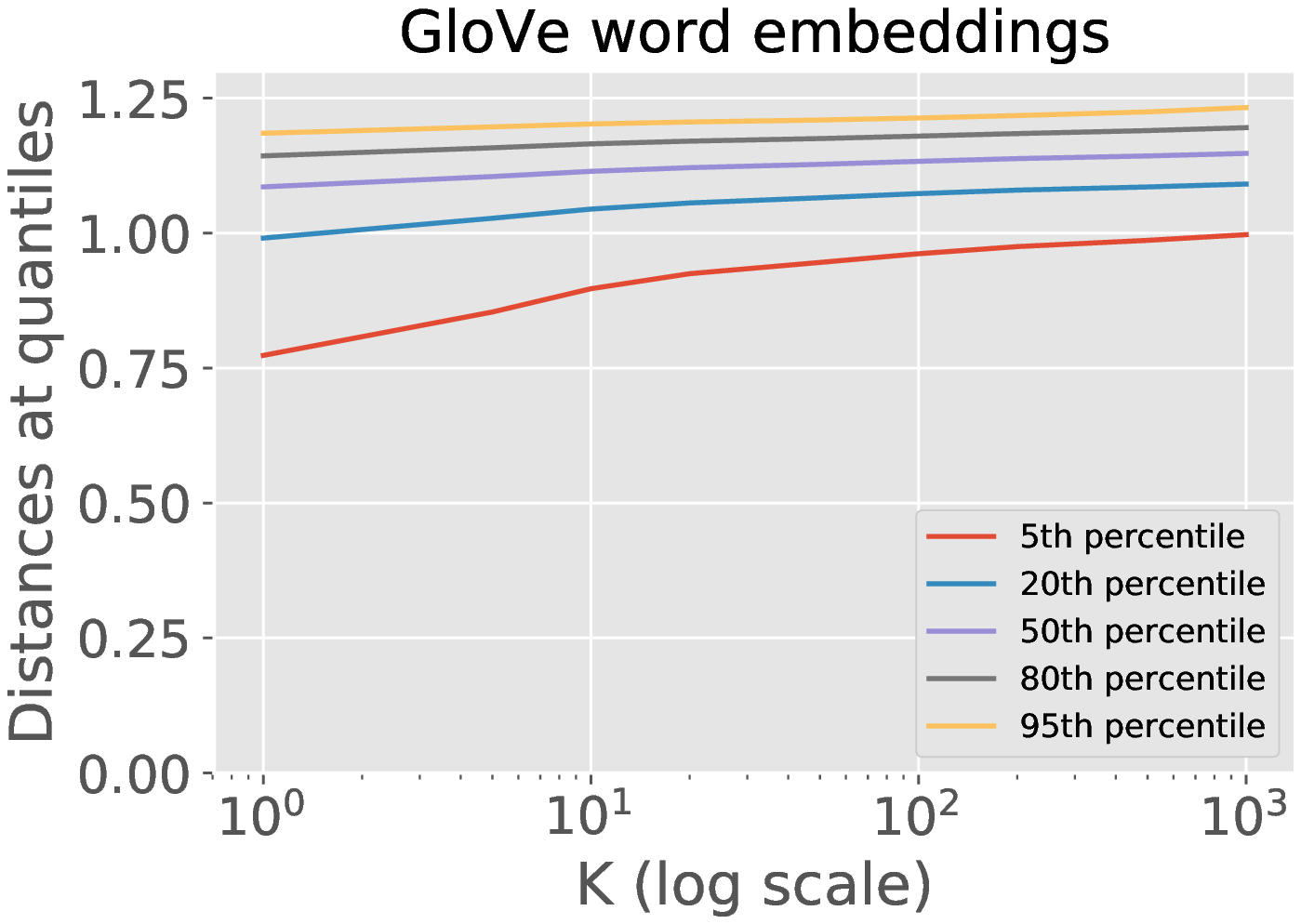}
    \includegraphics[width=.47\linewidth]{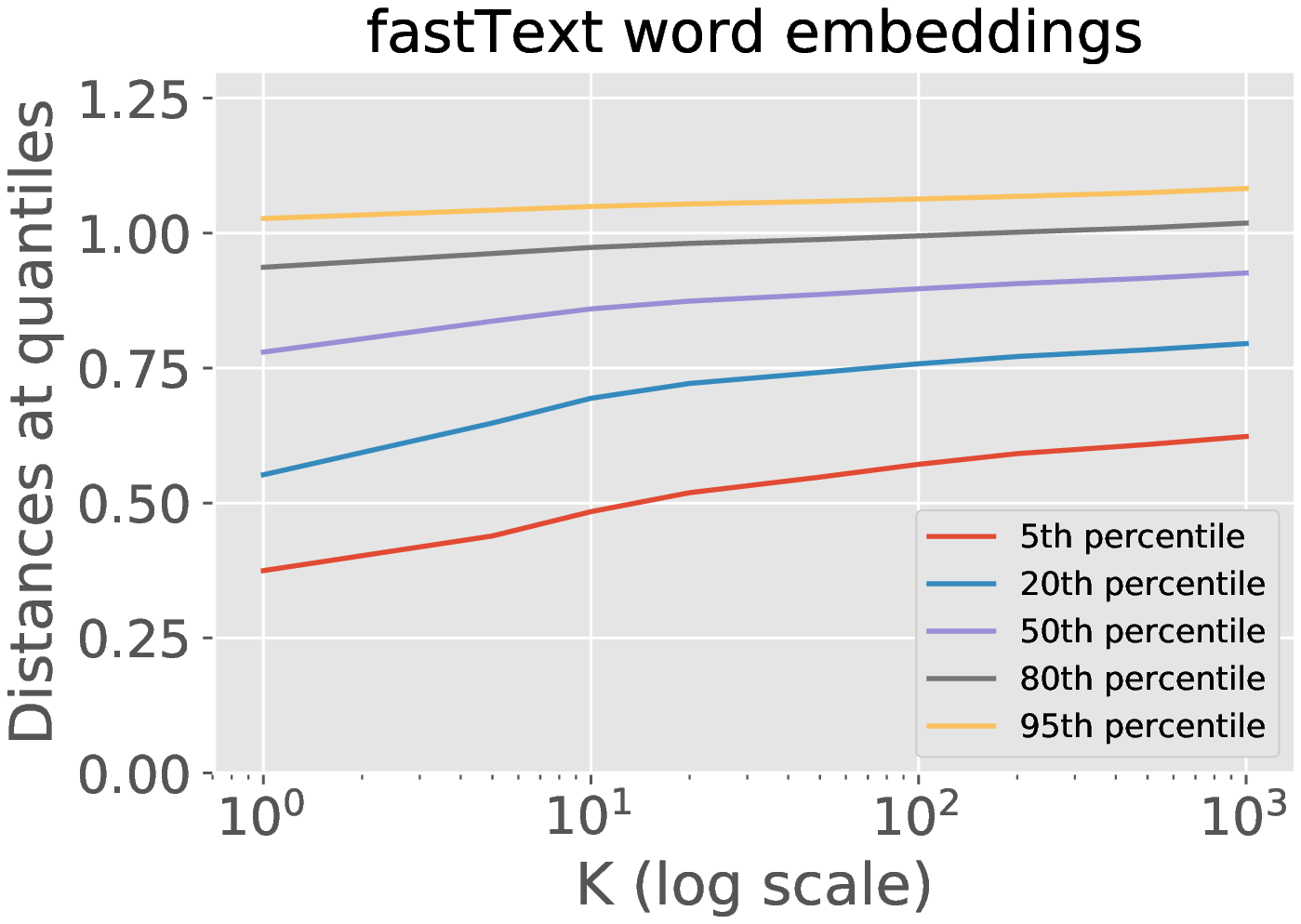}
\end{subfigure}
\caption{Distribution of distances between a given vector and its $k$ closest neighbors for \textsc{GloVe} and \textsc{fastText}}
\label{fig:emb_summary}
\end{figure}

The line plot results in \Cref{fig:emb_summary} give insights into how different embedding models of the same vector dimension can have different distance distributions. The words in \textsc{fastText} have a smoother distance distribution with a wider spread across percentiles. %The histograms in the figure (reported for $k = 10$) shows that on average, word in \textsc{fastText} will be closer to their $k$-th nearest neighbor than in \textsc{GloVe}. These results can be attributed to the difference between the size of the vocabularies in both embedding models.

\begin{figure*}
\begin{subfigure}[t]{\textwidth}
    \includegraphics[width=.205\linewidth]{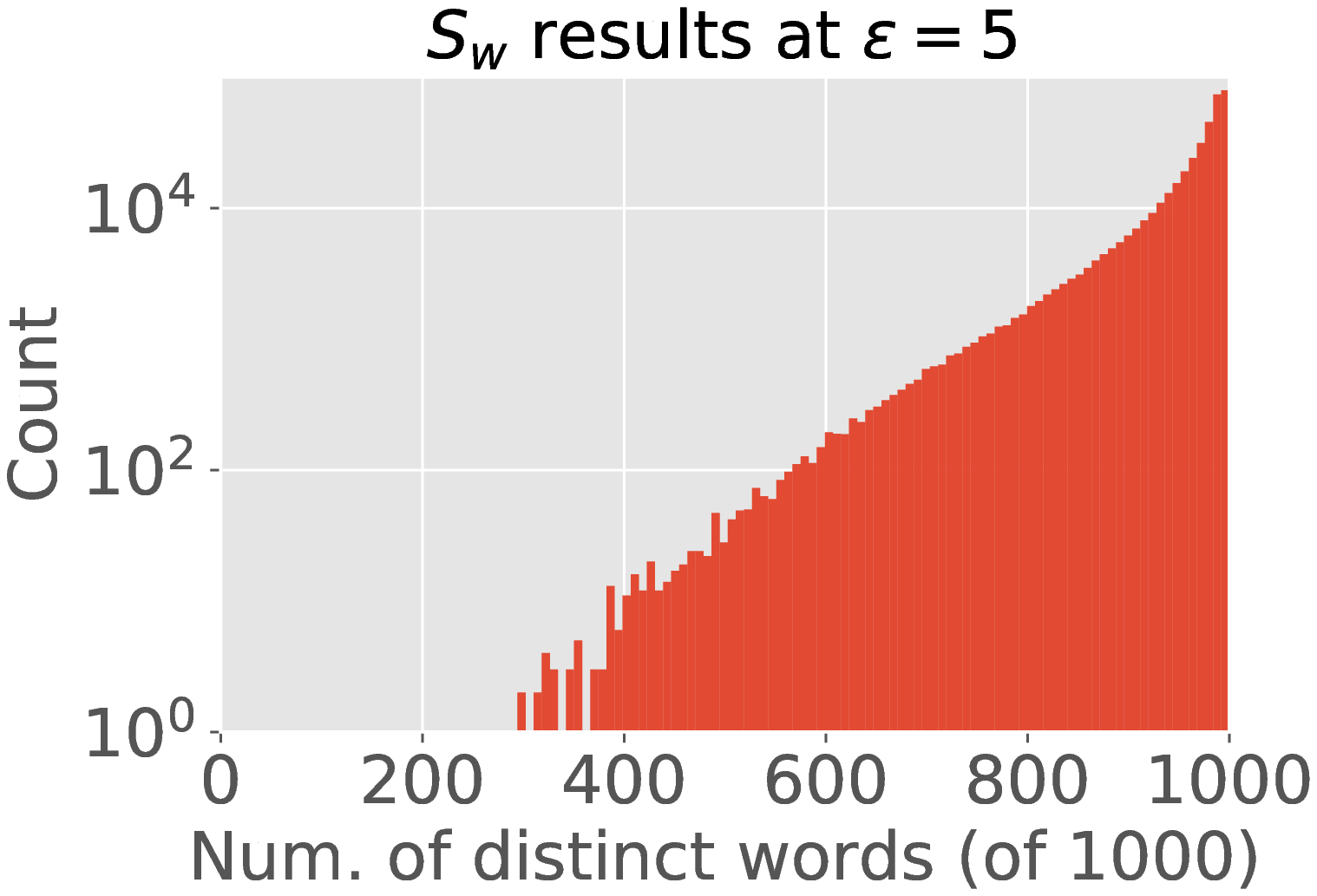}
    \includegraphics[width=.19\linewidth]{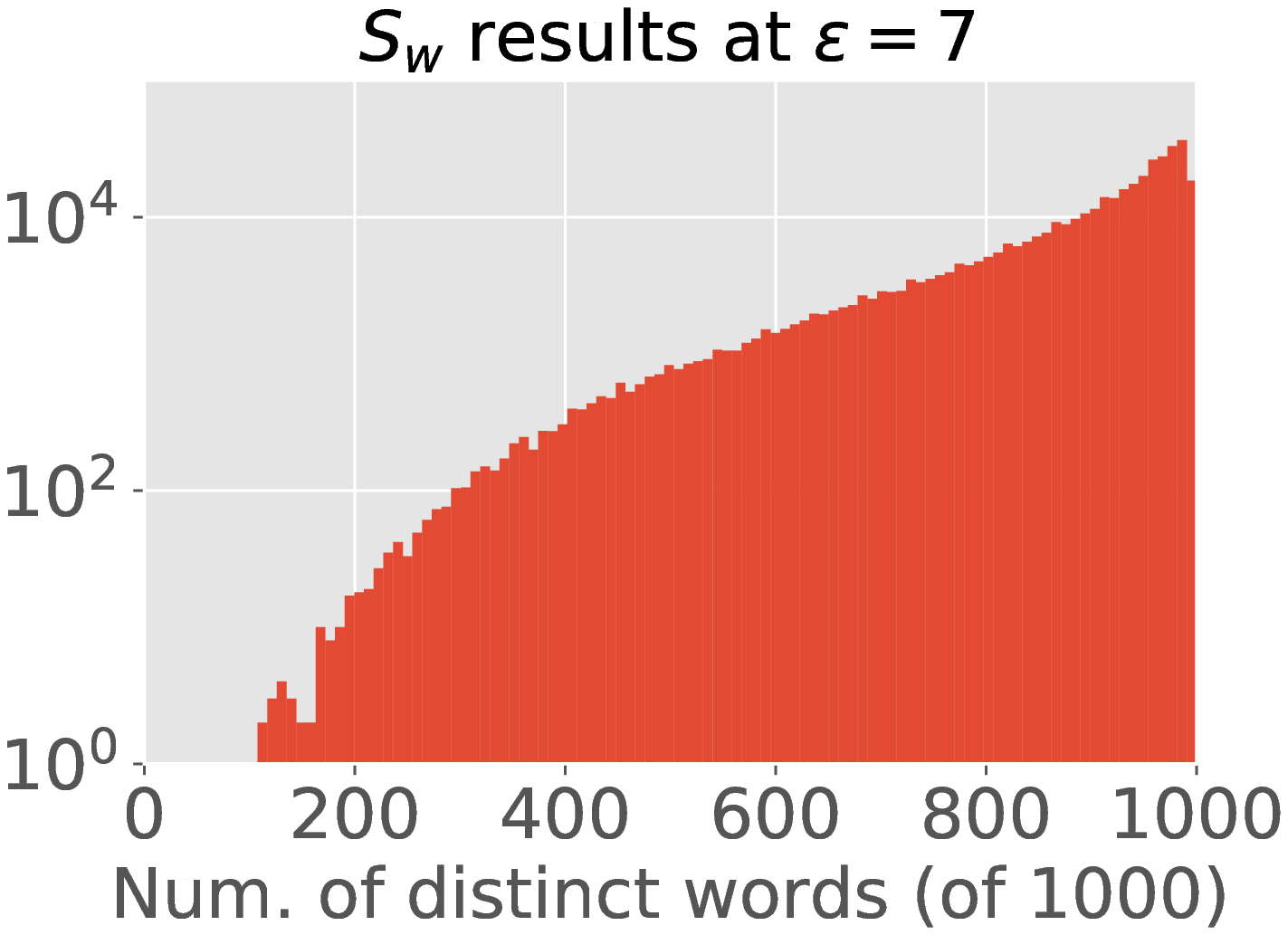}
    \includegraphics[width=.19\linewidth]{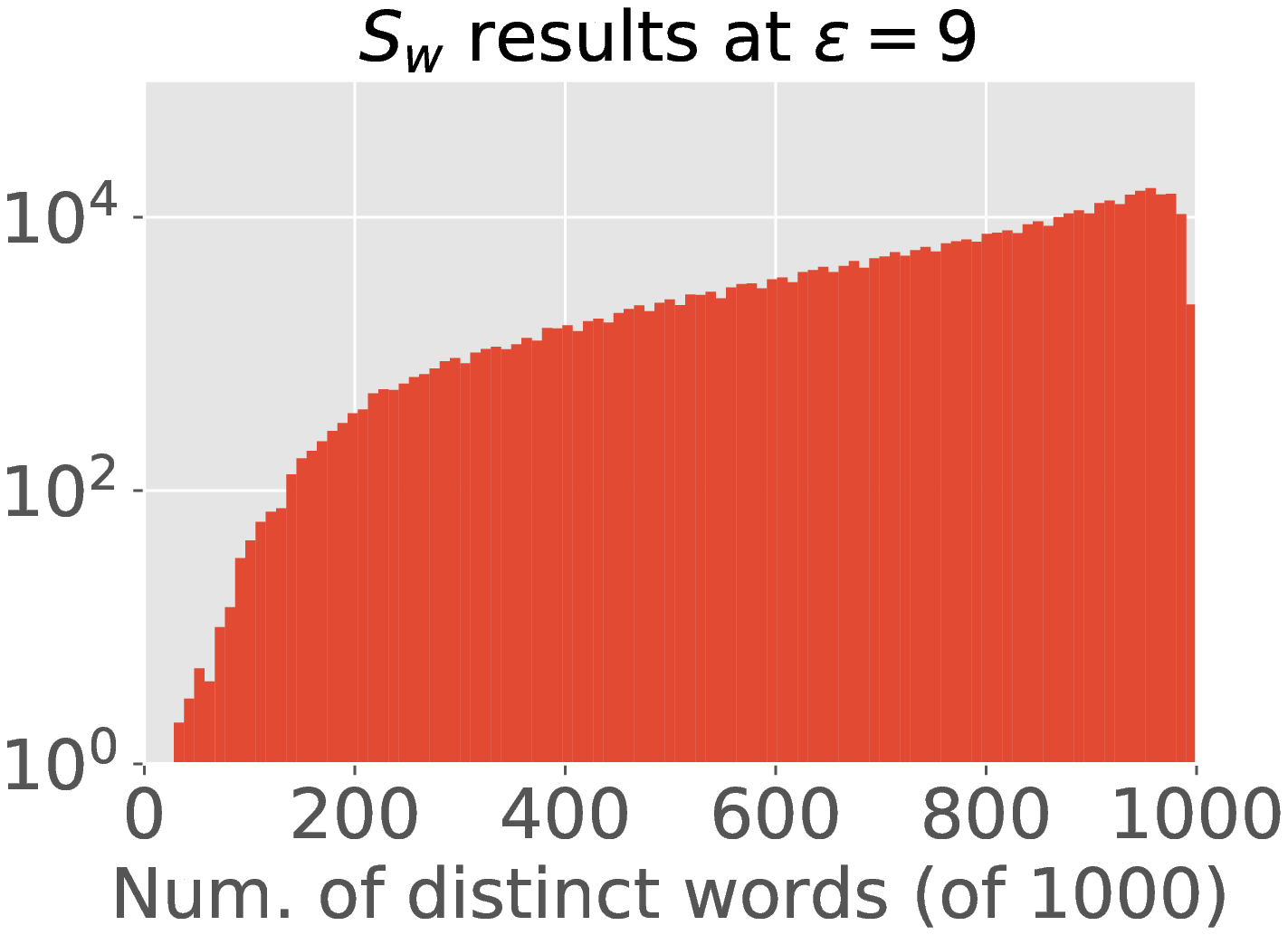}
    \includegraphics[width=.19\linewidth]{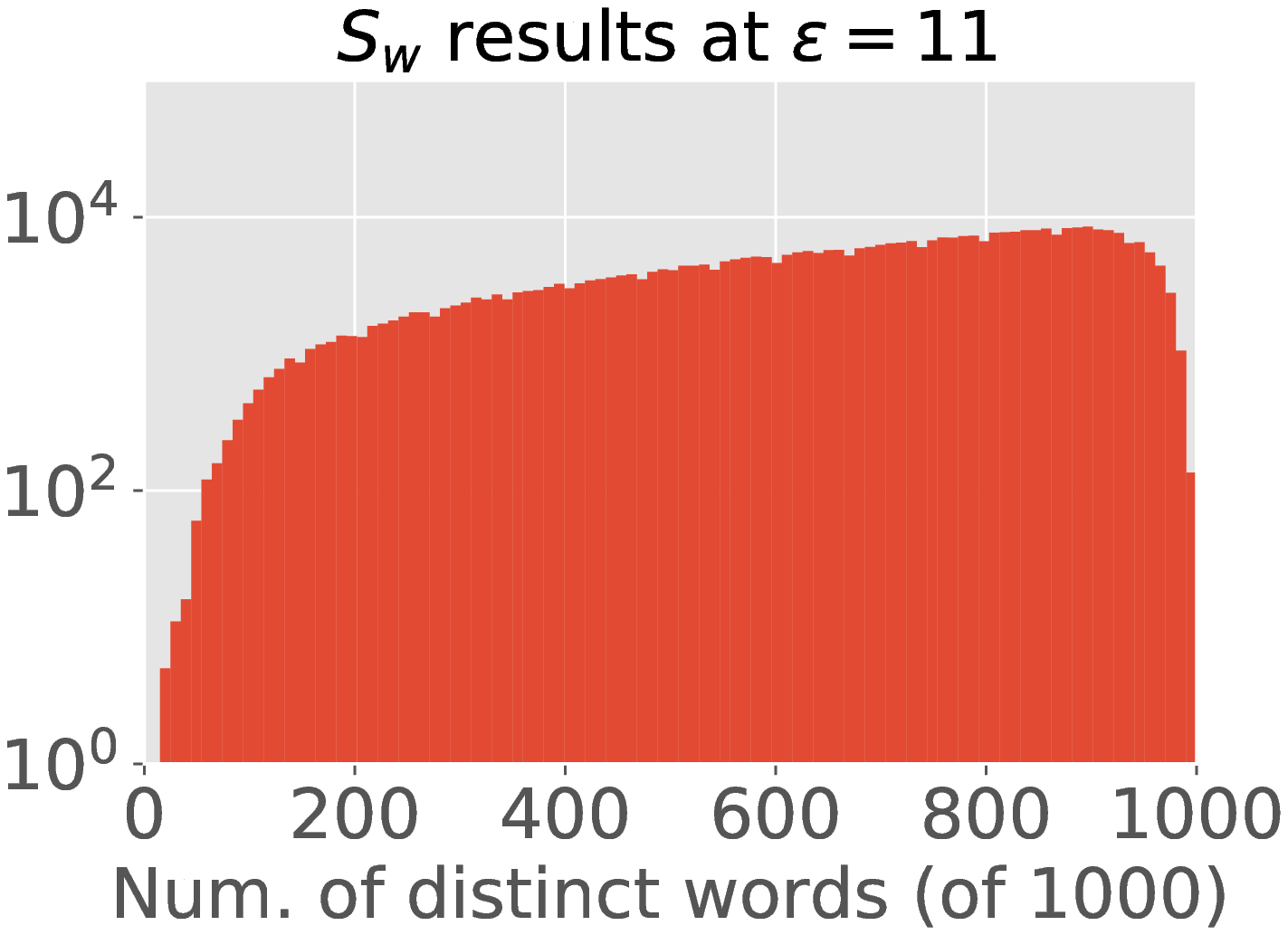}
    \includegraphics[width=.19\linewidth]{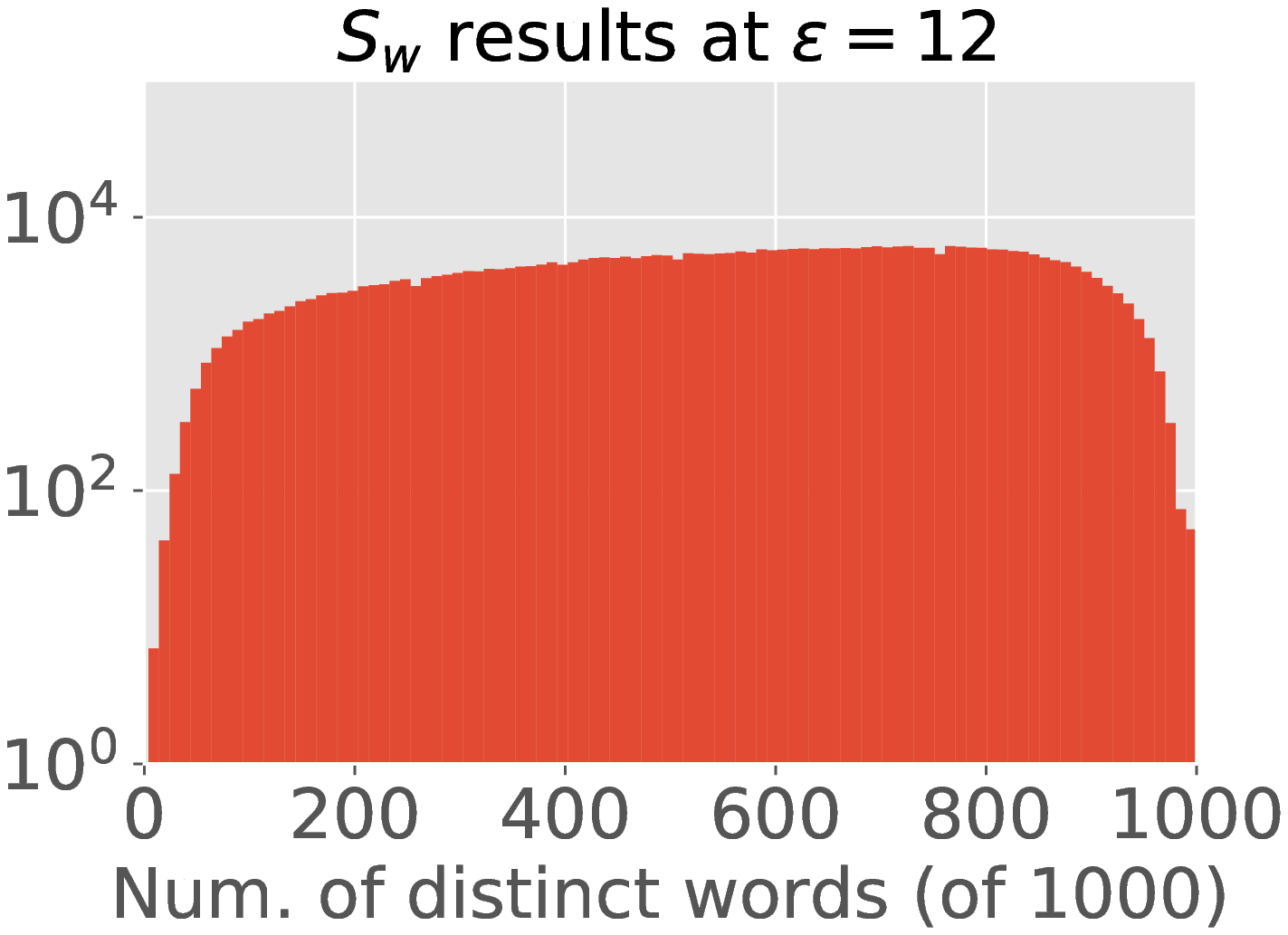}
\end{subfigure}
\begin{subfigure}[t]{\textwidth}
    \includegraphics[width=.205\linewidth]{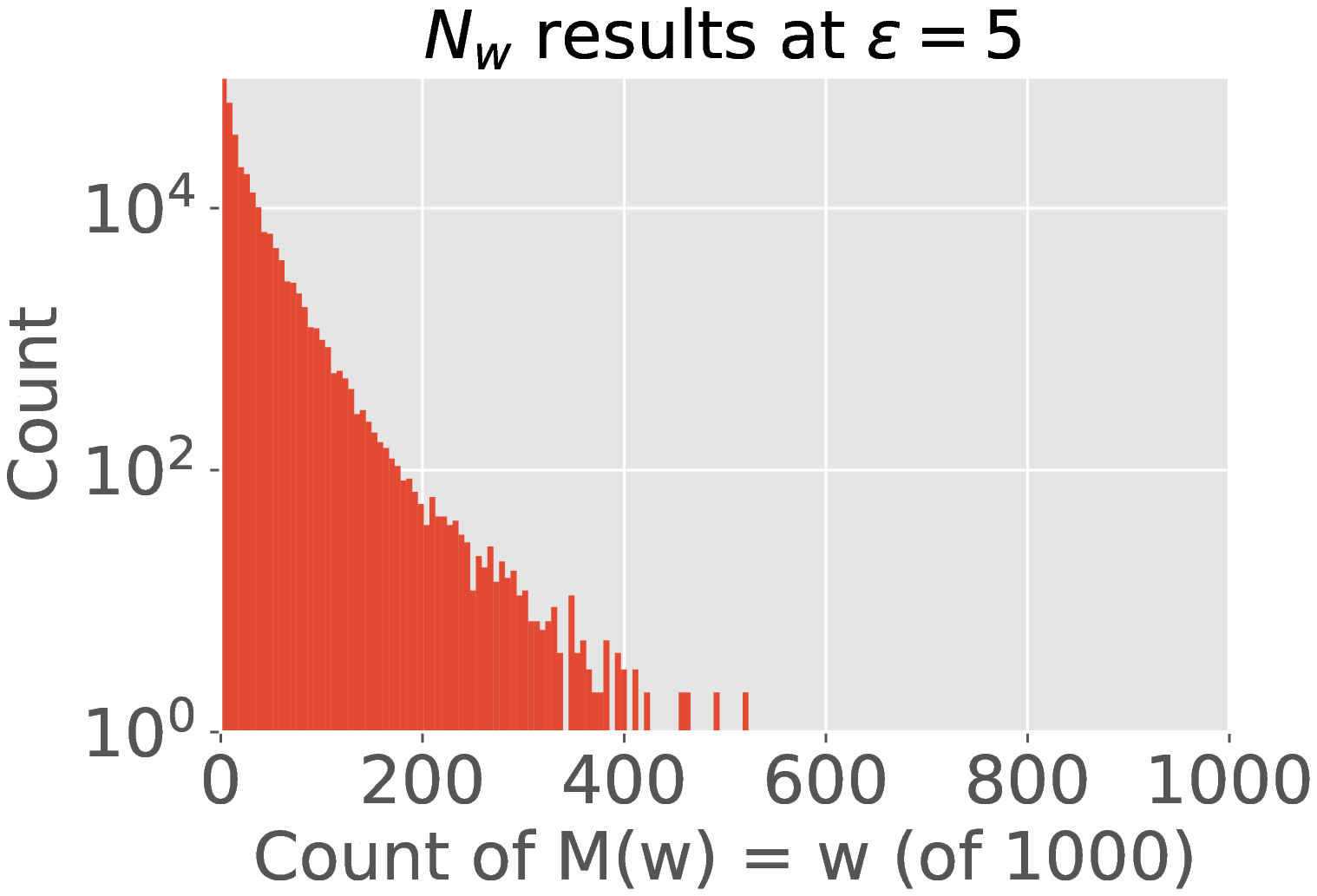}
    \includegraphics[width=.19\linewidth]{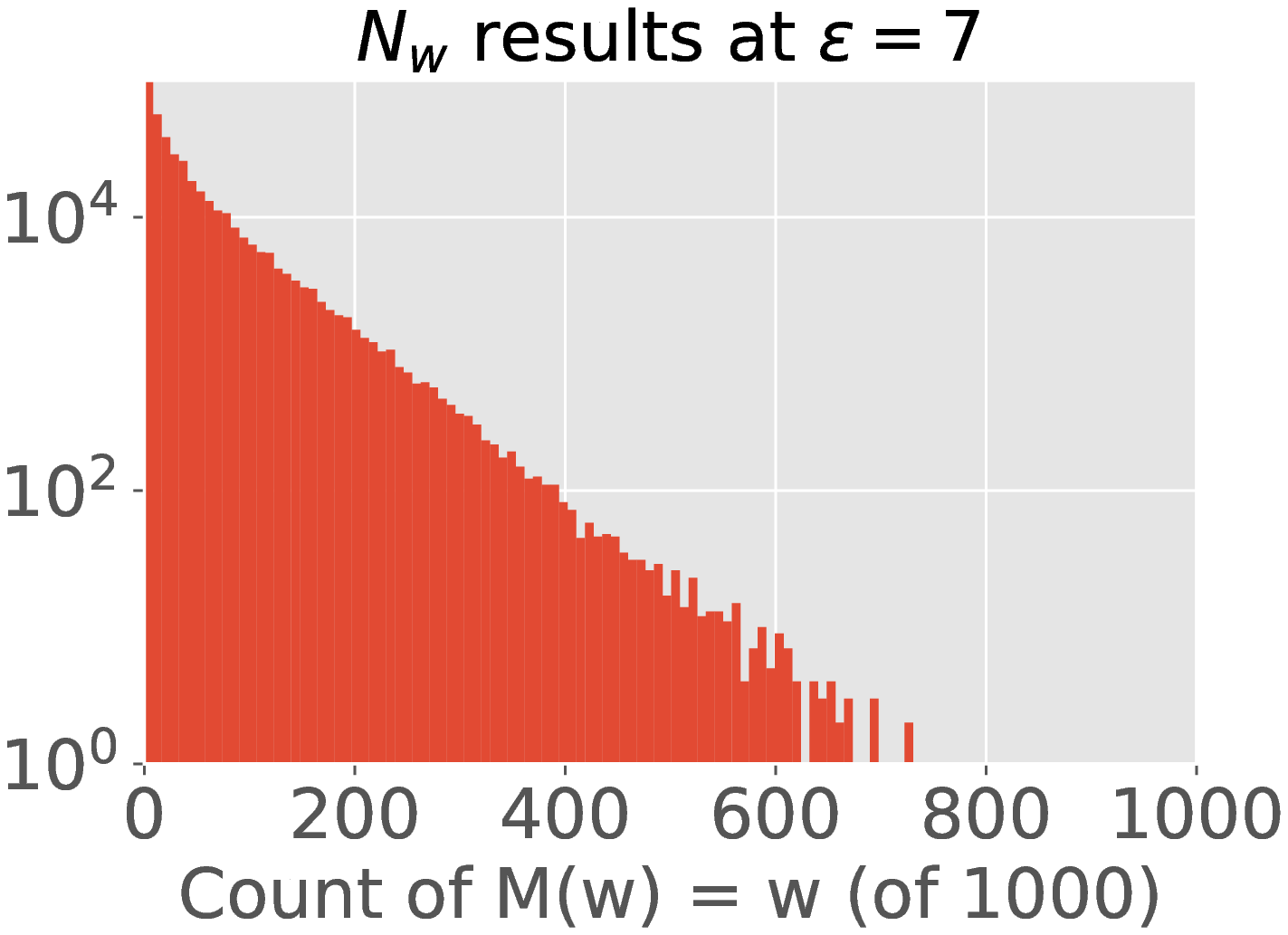}
    \includegraphics[width=.19\linewidth]{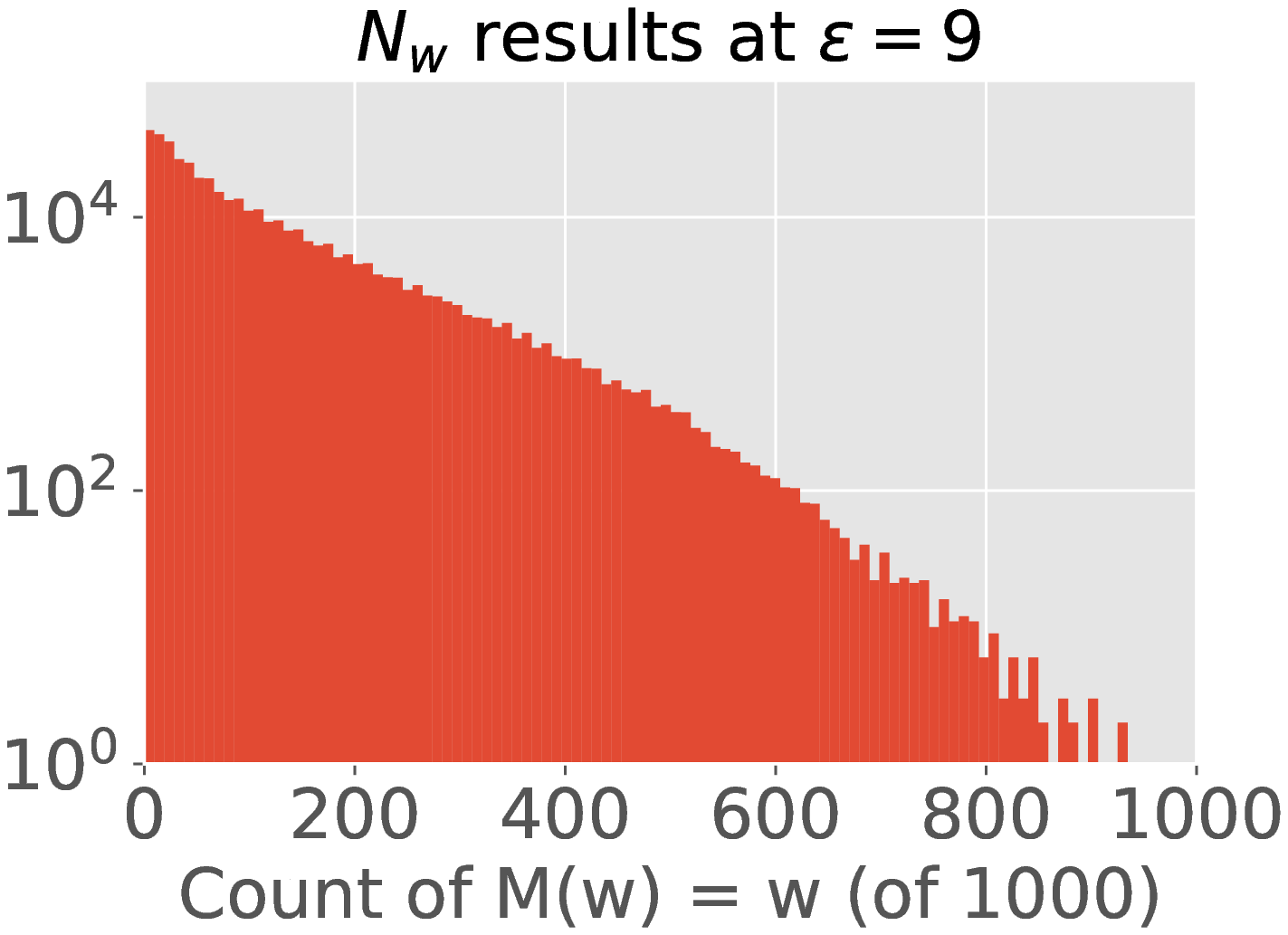}
    \includegraphics[width=.19\linewidth]{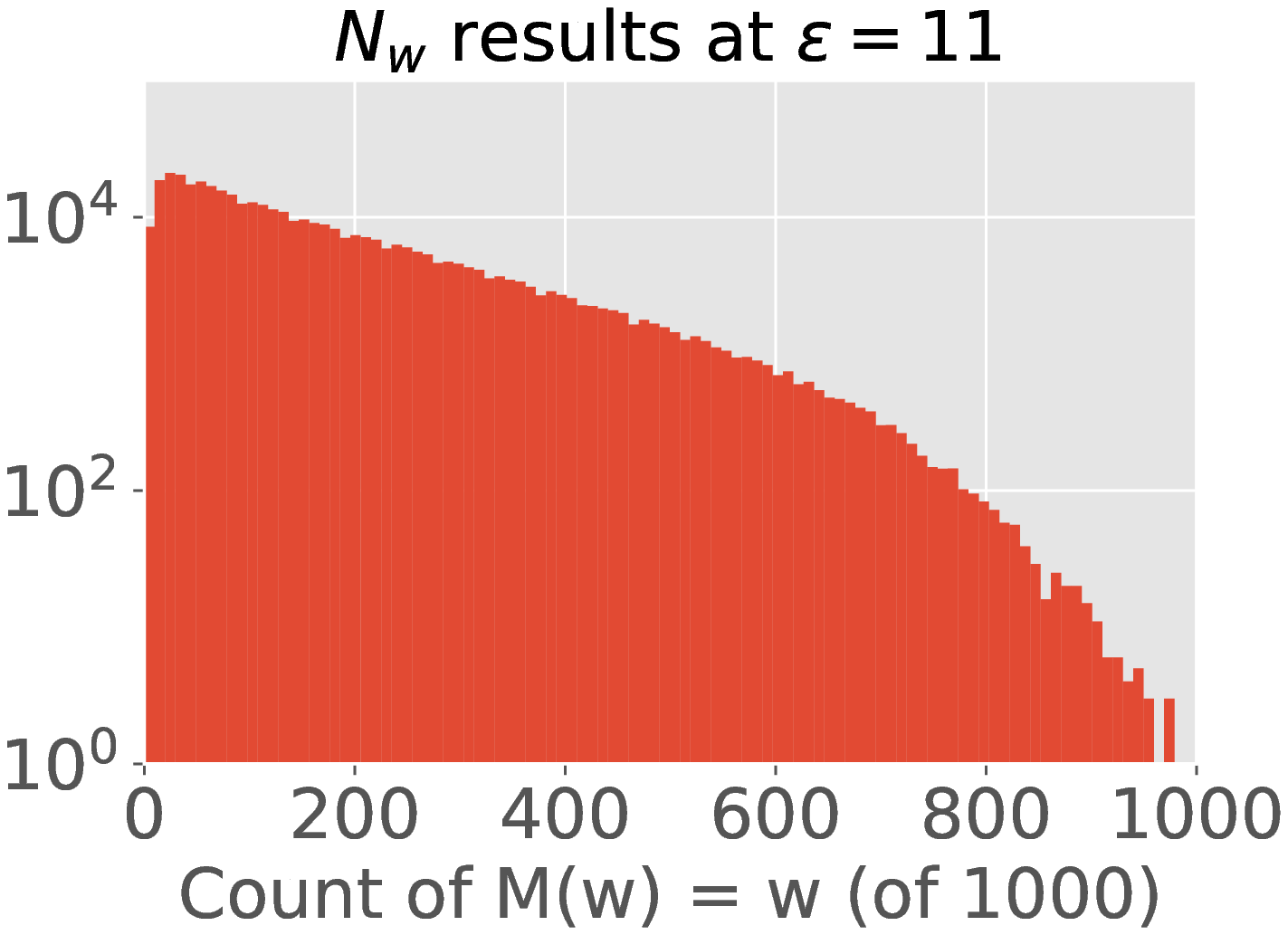}
    \includegraphics[width=.19\linewidth]{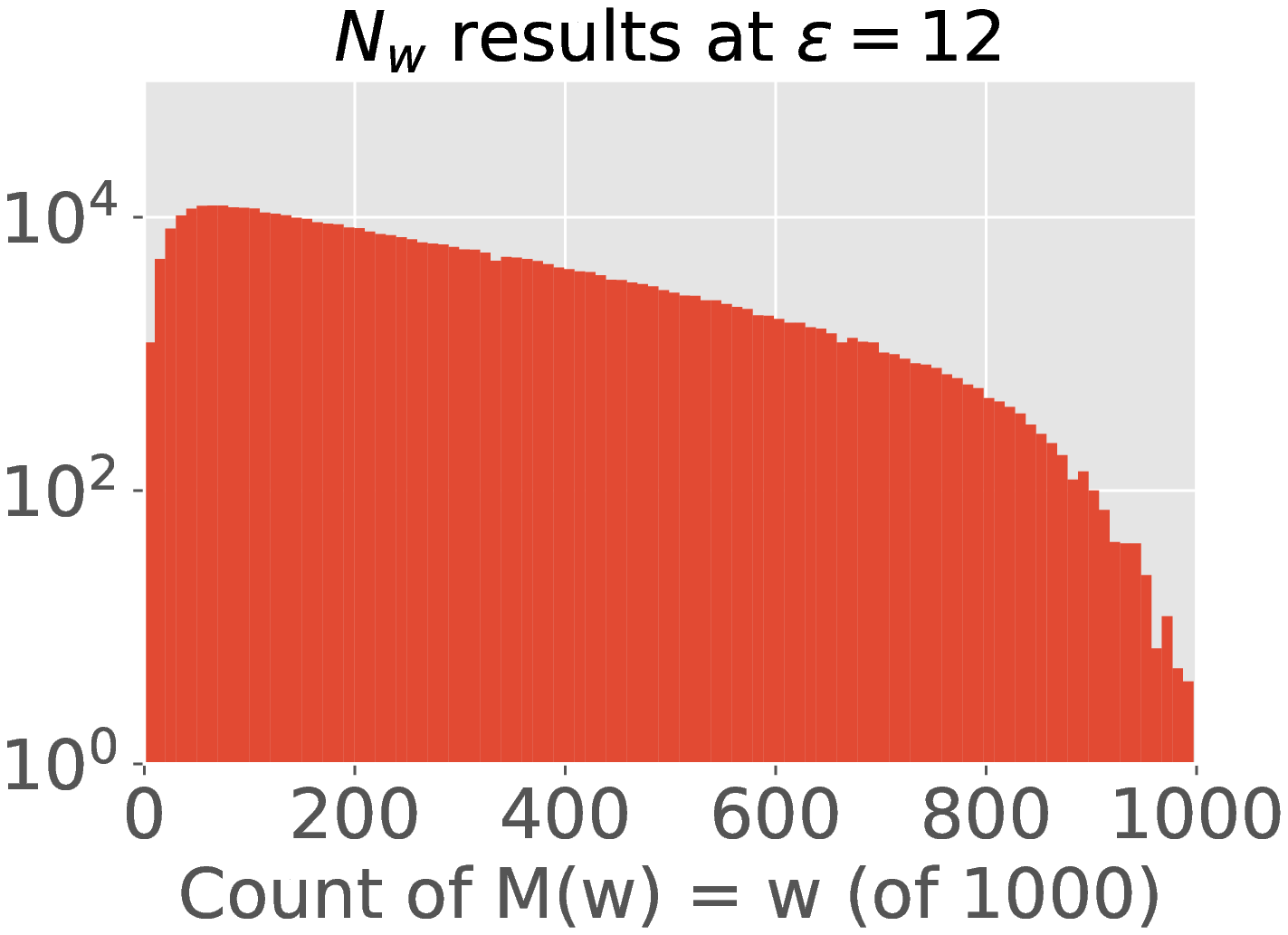}
\end{subfigure}
\caption{Empirical $S_w$ and $N_w$ statistics for $50$ dimensional \textsc{GloVe} word embeddings as a function of $\varepsilon$.}
\label{fig:tune-epsilon-glove50d}
\end{figure*}

\subsection{Word Distribution Statistics}
\label{sec:word_distribution}
We ran the mechanism $1,000$ times on input $w$ to compute the plausible deniability statistics $S_w$ and $N_w$ at different values of $\varepsilon$ for each word embedding model. For each word $w$ and the corresponding list of $1,000$ new words $W'$ from our $d_{\chi}$ perturbation, we recorded: (i) the probability $N_w = \Pr[M(w) = w]$ of not modifying the input word $w$ (estimated as the empirical frequency of the event $M(w) = w$); and (ii) the (effective) support of the output distribution $S_w$ (estimated by the distinct words in $W'$). 

%The results presented in \Cref{fig:tune-epsilon-glove50d} provide a visual way of selecting $\varepsilon$ for task types of different sensitivities. For example, using $S_w$, we can choose a desired average number of possible replacement words from our algorithm, then observe the shape of the histogram distribution at the value of $\varepsilon$ that achieves this average $S_w$ value. On the other hand, a possible approach for picking $\varepsilon$ based on $N_w$ can take the following form: we can inspect the histogram results at $\varepsilon = \gloveScale{500}$ and $\varepsilon = \gloveScale{600}$ and make the following observations of over $N_w$. At $\varepsilon = \gloveScale{500}$, no result set yields a support set where the only candidate for replacement was the original word (\ie{} where all the $1,000$ words returned from the algorithm was equal to $w$). This transition happens at $\varepsilon = \gloveScale{600}$ where we first observe this phenomenon (\ie{} the histogram bars reach the end of the x-axis at value $1,000$). Therefore, by looking at the distribution of $S_w$ and $N_w$ over different value of $\varepsilon$, we can make a principled choice on how to select $\varepsilon$ for a given embedding model.

The results presented in \Cref{fig:tune-epsilon-glove50d} provide a visual way of selecting $\varepsilon$ for task types of different sensitivities. We can select appropriate values of $\varepsilon$ by selecting our desired \emph{worst case} guarantees, then observing the extreme values of the histograms for $N_w$ and $S_w$. For example, at $\varepsilon = 5$, no word yields fewer than $300$ distinct new words ($S_w$ graph), and no word is ever returned more than $500$ times in the worst case ($N_w$ graph). Therefore, by looking at the worst case guarantees of $S_w$ and $N_w$ over different values of $\varepsilon$, we can make a principled choice on how to select $\varepsilon$ for a given embedding model.

\begin{table*}[h]
\footnotesize
\begin{center}
\begin{tabular}{| c | l | l | l | l | l | l | l | l | }
 \hline
&& \multicolumn{2}{c|}{\textbf{\emph{w} = encryption}} & \multicolumn{2}{c|}{\textbf{\emph{w} = hockey}} & \multicolumn{2}{c|}{\textbf{\emph{w} = spacecraft}}\\
 \hline
$\varepsilon$ &Avg. $N_w$  & \textsc{GloVe} & \textsc{fastText} & \textsc{GloVe} & \textsc{fastText} & \textsc{GloVe} & \textsc{fastText} \\ 
 \hline
 \hline
\parbox[t]{2mm}{\multirow{17}{*}{\rotatebox[origin=c]{90}{$\longleftarrow$ increasing $\varepsilon$, better semantics}}}

&$50$ & freebsd & ncurses & stadiumarena & dampener & telemeter & geospace \\
&& multibody & vpns & futsal & popel & deorbit & powerup \\ 
&& 56-bit & tcp & broomball & decathletes & airbender & skylab \\
&& public-key & isdn & baseballer & newsweek & aerojet & unmanned \\
\cline{2-8}
&$100$ & ciphertexts & plaintext & interleague & basketball & laser & voyager \\
&& truecrypt & diffie-hellman & usrowing & lacrosse & apollo-soyuz & cassini-huygens \\
&& demodulator & multiplexers & football & curlers & agena & adrastea \\
&& rootkit & cryptography & lacrosse & usphl & phaser & intercosmos \\
\cline{2-8}
&$200$ &harbormaster & cryptographic & players & goaltender & launch & orbited \\
&& unencrypted & ssl/tls & ohl & ephl & shuttlecraft & tatooine \\
&& cryptographically & authentication & goaltender & speedskating & spaceborne & flyby \\
&& authentication & cryptography & defenceman & eishockey & interplanetary & spaceborne \\
\cline{2-8}
&$300$ & decryption & encrypt & nhl & hockeygoalies & spaceplane & spaceship \\
&& encrypt & unencrypted & hockeydb & hockeyroos & spacewalk & spaceflights \\
&& encrypted & encryptions & hockeyroos & hockeyettan & spaceflights & satellites \\
&& encryption & encrypted & hockey & hockey & spacecraft & spacecraft \\
 \hline

\hline
\end{tabular}
\caption{Output $\hat{w} = M(w)$ on topic model words from the 20 Newsgroups dataset. Selected words $w$ are from \cite{larochelle2012neural} } \label{tab:all_examples}
\end{center}
\end{table*}

\subsection{Selecting Between Different Embeddings}
Our analysis gives a reasonable approach to selecting $\varepsilon$ (i.e., via \emph{worst case} guarantees) by means of the proxies provided by the plausible deniability statistics. In general, tuning privacy parameters in $d_{\chi}$-privacy is still a topic under active research \cite{hsu2014differential}, especially with respect to what $\varepsilon$ means for different applications. %Task owners ultimately need to determine what is best based on the available descriptive statistics. 

\begin{figure}[h]
\begin{subfigure}[t]{\linewidth}
    \includegraphics[width=.51\linewidth]{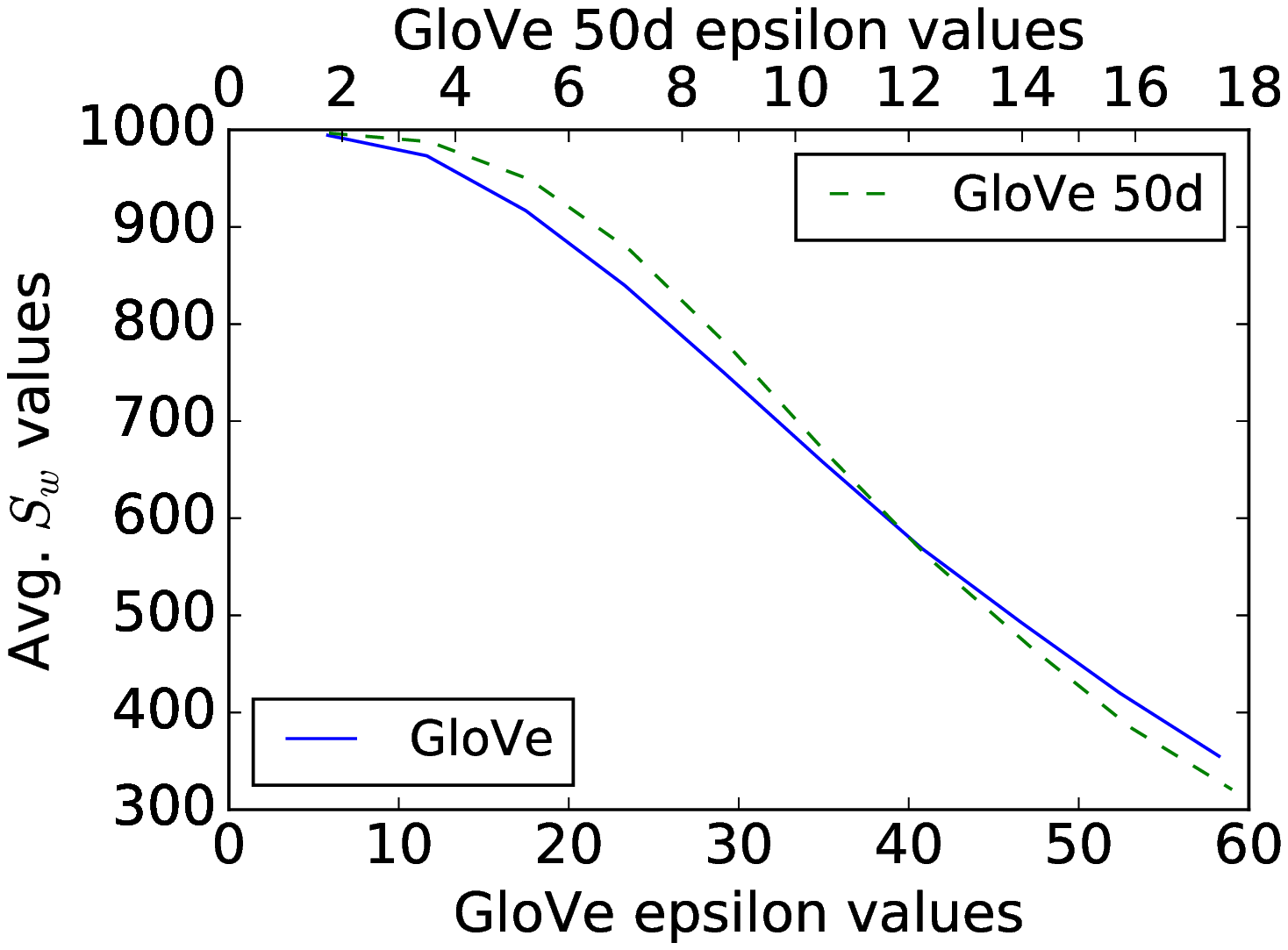}
    \includegraphics[width=.497\linewidth]{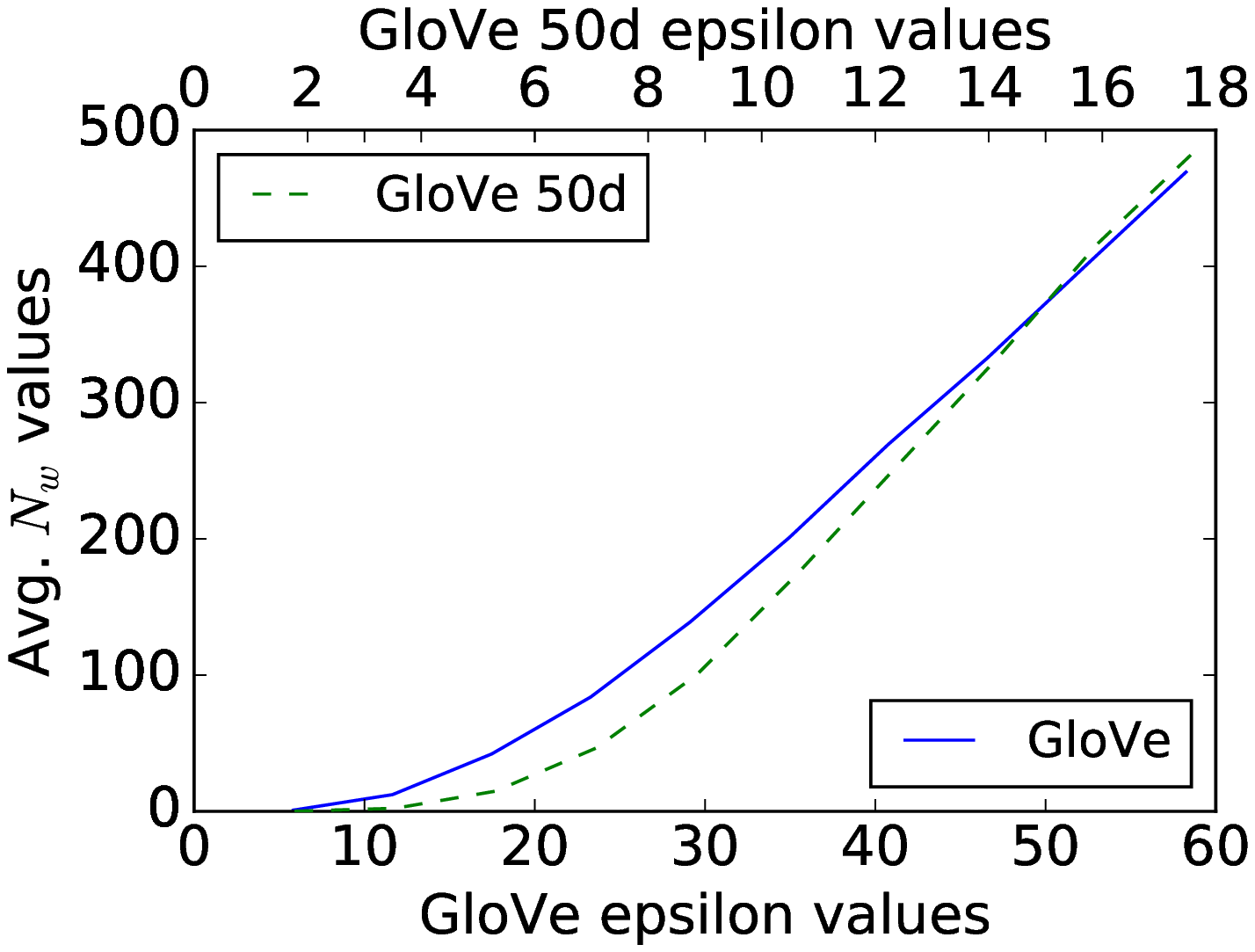}
\end{subfigure}
\caption{Average $S_w$ and $N_w$ statistics: \textsc{GloVe} $50d$ and $300d$}
\label{fig:emb_glove_stat_compare}
\end{figure}

With regards to the same embedding model with different dimensionalities, \Cref{fig:emb_glove_stat_compare} suggests that they do provide the same level of average case guarantees (at `\emph{different}' values of $\varepsilon$). Therefore, selecting a model becomes a function of utility on downstream tasks. \Cref{fig:emb_glove_stat_compare} further underscores the need to interpret the notion of $\varepsilon$ in $d_{\chi}$-privacy within the context of the metric space.

\begin{figure}[h]
\begin{subfigure}[t]{\linewidth}
    \includegraphics[width=.51\linewidth]{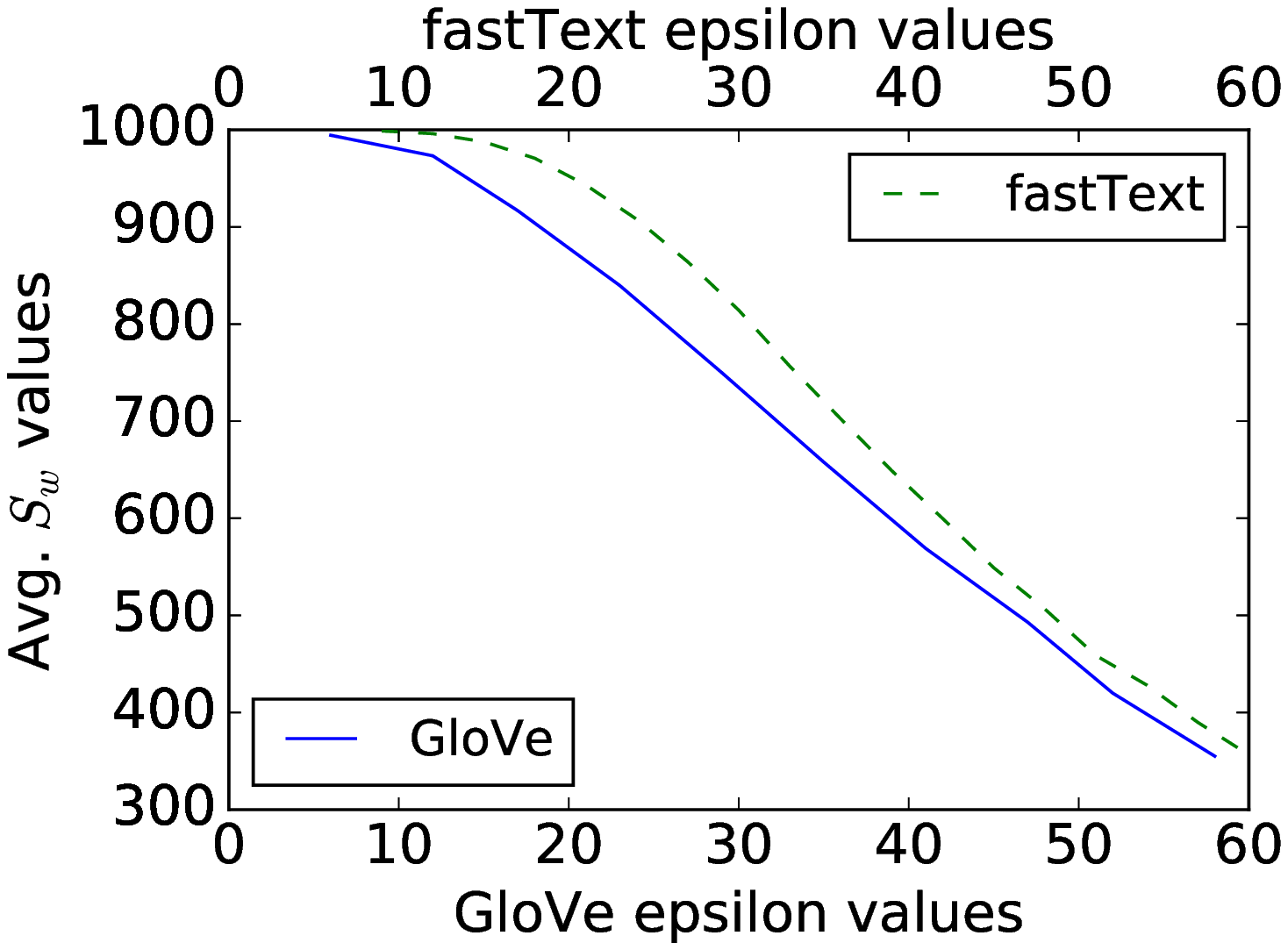}
    \includegraphics[width=.497\linewidth]{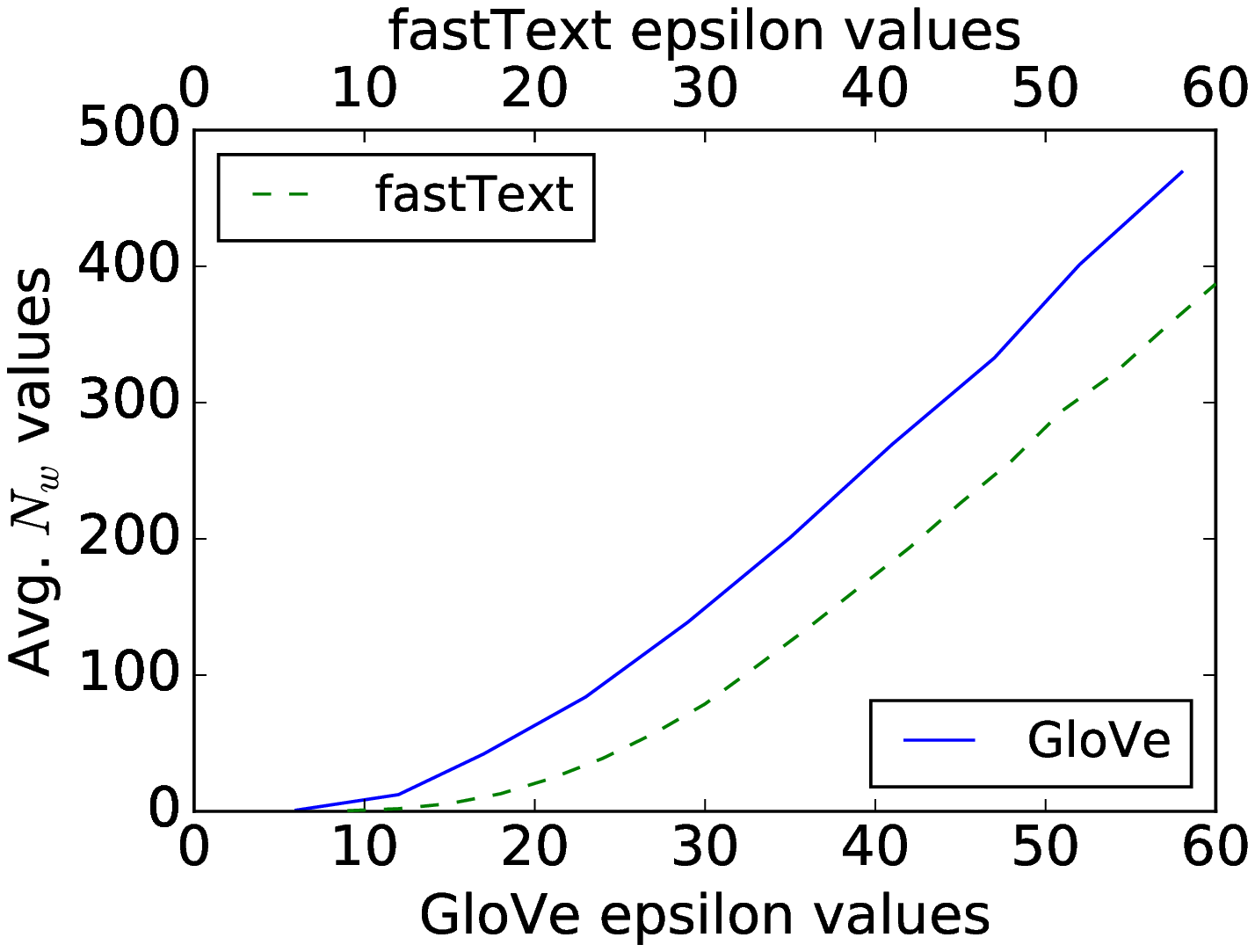}
\end{subfigure}
\caption{Average $S_w$ and $N_w$ statistics: \textsc{GloVe} and \textsc{fastText}}
\label{fig:emb_stat_compare}
\end{figure}

In \Cref{fig:emb_stat_compare}, we present the average values of $S_w$ and $N_w$ statistics for \textsc{GloVe} and \textsc{fastText}. 
%We observe that similar values over fastText cover a broader range of $\varepsilon$ values when compared to GloVe. 
However, the average values are not sufficient to make a conclusive comparison between embedding models since different distributions can result in the same entropy -- therefore, we recommend setting worst case guarantees. For further discussions on the caveats of interpreting entropy based privacy metrics see \cite{wagner2018technical}.

\Cref{tab:all_examples} presents examples of word perturbations on similar mechanisms calibrated on \textsc{GloVe} and \textsc{fastText}. The results show that as the average values of $N_w$ increase (corresponding to higher values of $\varepsilon$), the resulting words become more similar to the original word.

\section{\ac{ML} Utility Experiments}\label{sec:ml_experiments}
We describe experiments we carried out to demonstrate the trade-off between privacy and utility for three downstream \ac{NLP} tasks.

\subsection{Datasets}
We ran experiments on three textual datasets, each representing a common task type in \ac{ML} and \ac{NLP}. The datasets are: IMDb movie reviews (binary classification) \cite{maas2011learning}, Enron emails (multi-class classification) \cite{klimt2004enron}, and InsuranceQA (question answering) \cite{feng2015applying}. Each dataset contains contributions from individuals making them suitable dataset choices. \Cref{tab:dataset_summary} presents a summary of the datasets.

\begin{table}[h]
\smaller
\begin{tabular}{lllll}
\toprule
Dataset & IMDb & Enron & InsuranceQA\\
\midrule
Task type & binary  & multi-class  & QA \\
Training set size & $25,000$ & $8,517$ & $12,887$ \\
Test set size & $25,000$ & $850$ & $1,800$ \\
Total word count & $5,958,157$ & $307,639$ & $92,095$ \\
Vocabulary size & $79,428$ & $15,570$ & $2,745$ \\
Sentence length & \makecell[l]{$\mu = 42.27$ \\ $\sigma = 34.38$} & \makecell[l]{$\mu = 30.68$ \\ $\sigma = 31.54$} & \makecell[l]{$\mu = 7.15$ \\ $\sigma = 2.06$} \\
\bottomrule
\end{tabular}
\caption{Summary of selected dataset properties}
\label{tab:dataset_summary}
\end{table}

\subsection{Setup for utility experiments}
\label{sec:ml_setup}
For each dataset, we demonstrated privacy vs. utility at:

\textbf{Training time:} we trained the models on perturbed data, while testing was carried out on plain data. This simulates a scenario where there is more access to private training data.

\textbf{Test time:} here, we trained the models completely on the available training set. However, the evaluation was done on a privatized version of the test sets. 

%\end{itemize}

\subsection{Baselines for utility experiments}
%For all our experiments, we used $300$ dimension GloVe word embeddings in our deep learning models.
All models in our experiments use $300d$ \textsc{GloVe} embeddings (hence the seemingly \emph{larger} values of $\varepsilon$. See discourse in \Cref{sec:mdp} and \Cref{fig:emb_glove_stat_compare}) on \ac{biLSTM} models \cite{graves2005bidirectional}. %For statistics on $50$ dimensional GloVe embeddings, see Appendix~\ref{app:plausible_deniability_fastext}

\textbf{IMDb movie reviews}
%We built a \ac{biLSTM} model \cite{graves2005bidirectional} 
% TD commented - didn't read well IMO
%based on the model described here\footnote{Classify movie reviews -- \url{https://tools.google.com/seedbank/seed/5712536552865792}} 
%to classify the sentiment of the reviews.
The training set was split to as in \cite{maas2011learning}. The privacy algorithm was run on the partial training set of $15,000$ reviews. The evaluation metric used was classification accuracy. 

\textbf{Enron emails}
%We built a \ac{biLSTM} model with dropout and categorical cross entropy loss. 
Our test set was constructed by sampling a random $10$\% subset of the emails of $10$ selected authors in the dataset. The evaluation metric was also classification accuracy.

\textbf{InsuranceQA}
We replicated the results from \cite{tan2015lstm} with \ac{GESD} as similarity scores. The evaluation metrics used \ac{MAP} and \ac{MRR}.

The purpose of our baseline models was not to advance the state of the art for those specific tasks. They were selected to provide a standard that we could use to compare our further experiments. 

\subsection{Results for utility experiments}

\Cref{tab:dataset_summary} presents a high level summary of some properties of the $3$ datasets. It gives insights into the size of the vocabulary of each dataset, the total number of words present, the average length and standard deviation of sentences. The InsuranceQA dataset consisted of short one line questions within the insurance domain -- this is reflective in its smaller vocabulary size, shorter sentence length and small variance. The other $2$ datasets consisted of a broader vocabulary and wide ranging sentence structures.

\begin{figure*}[h]
\centering
\begin{subfigure}[t]{\textwidth}
\centering
    \includegraphics[width=.163\linewidth]{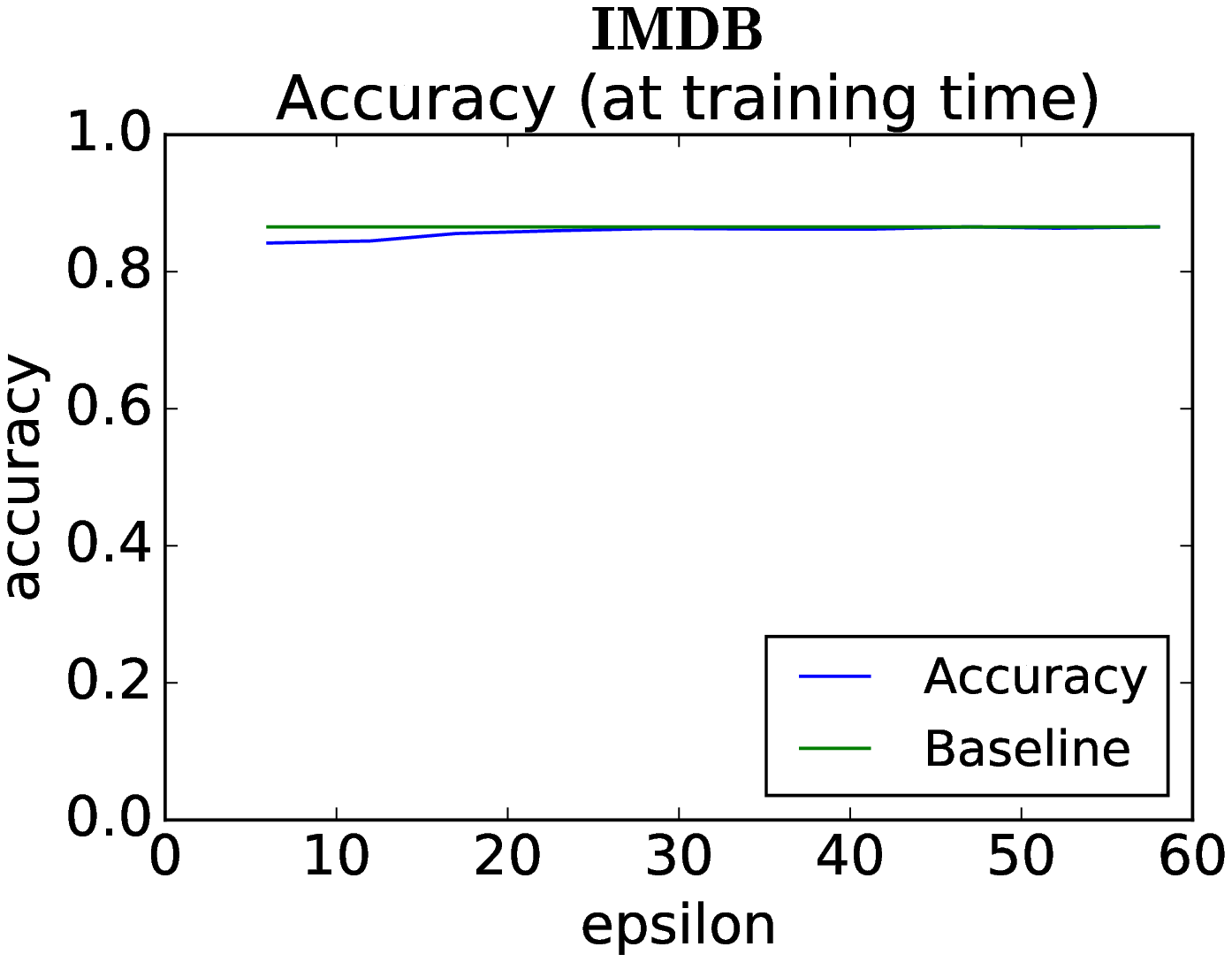}
    \includegraphics[width=.163\linewidth]{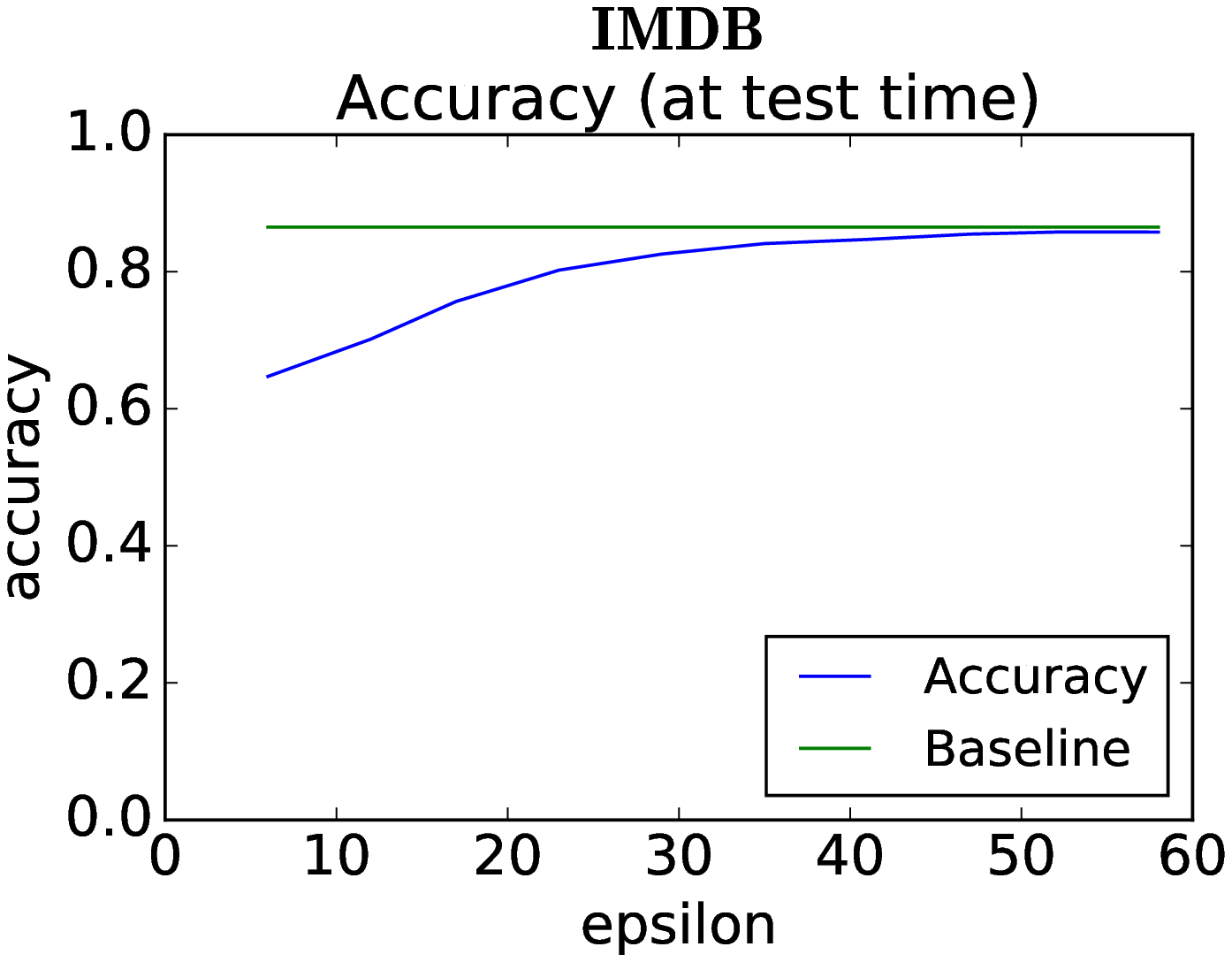}
    \includegraphics[width=.163\linewidth]{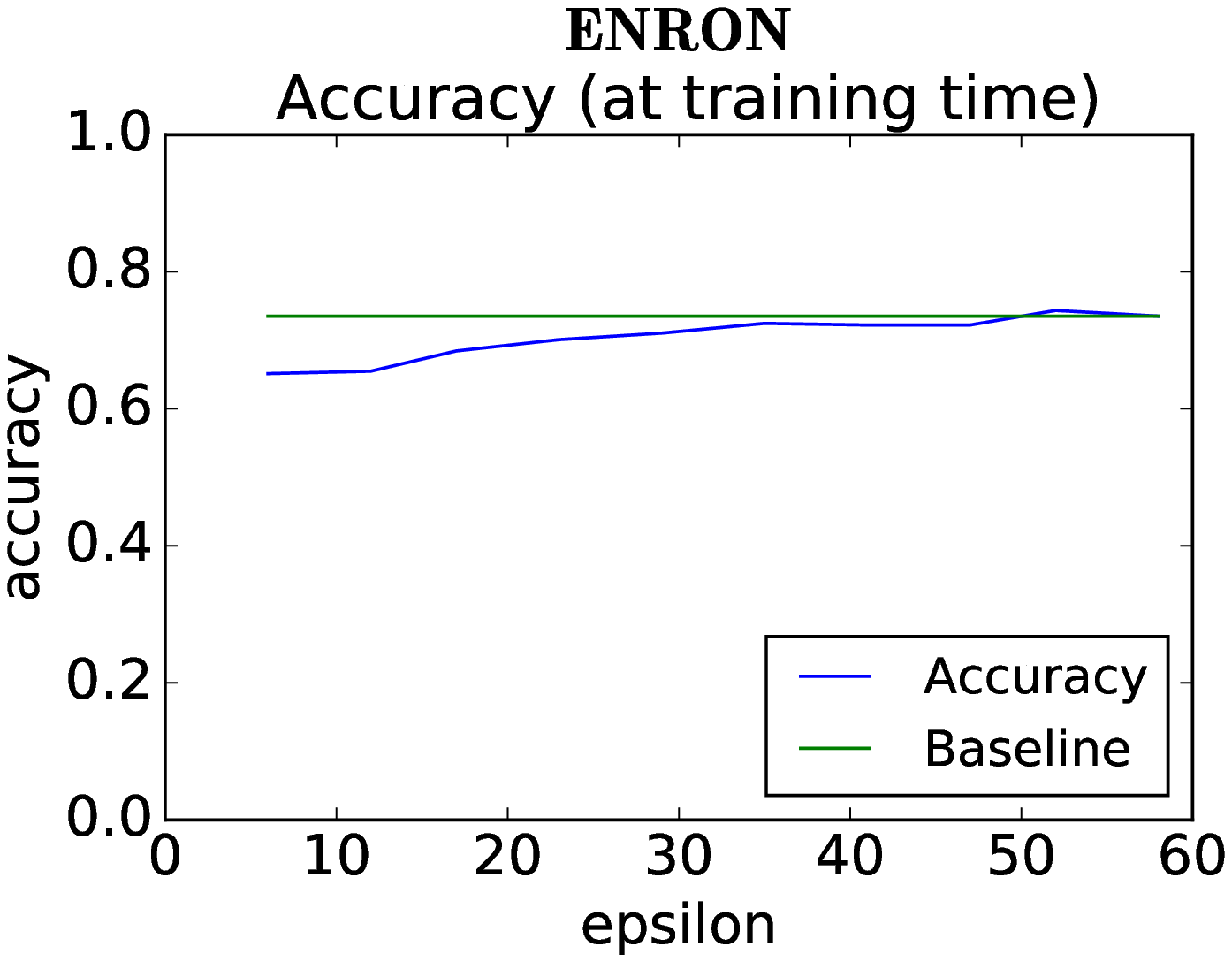}
    \includegraphics[width=.163\linewidth]{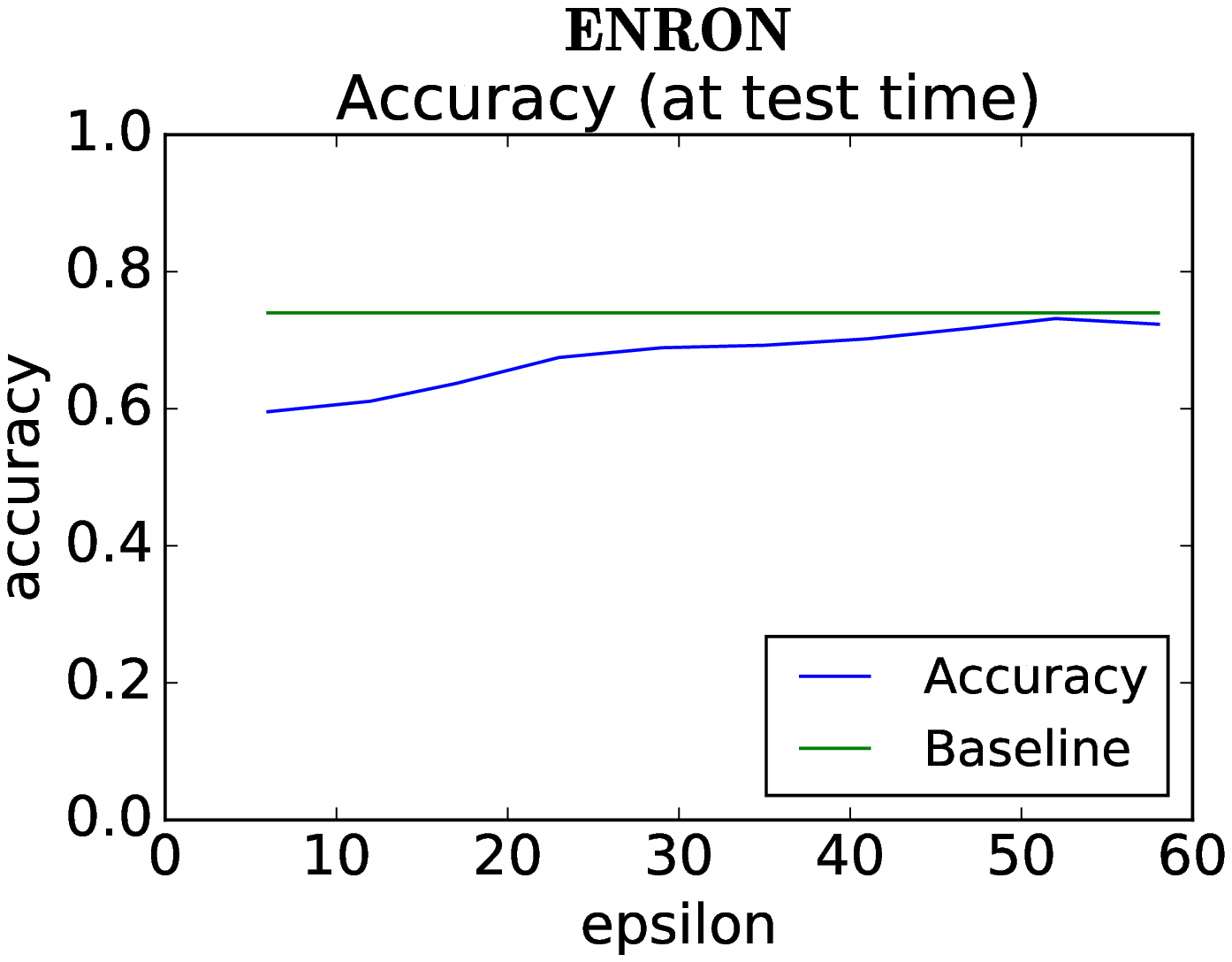}
    \includegraphics[width=.163\linewidth]{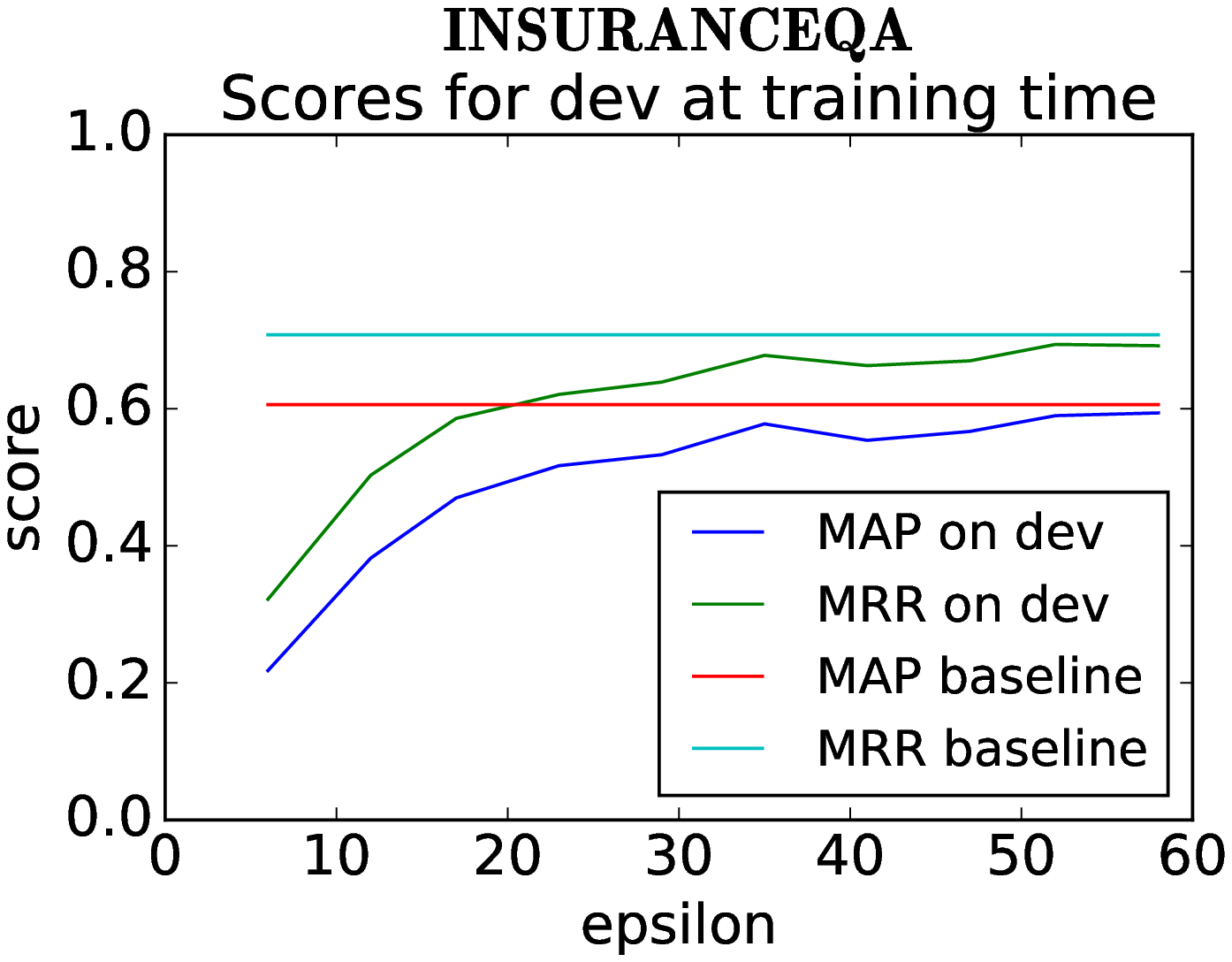}
    \includegraphics[width=.163\linewidth]{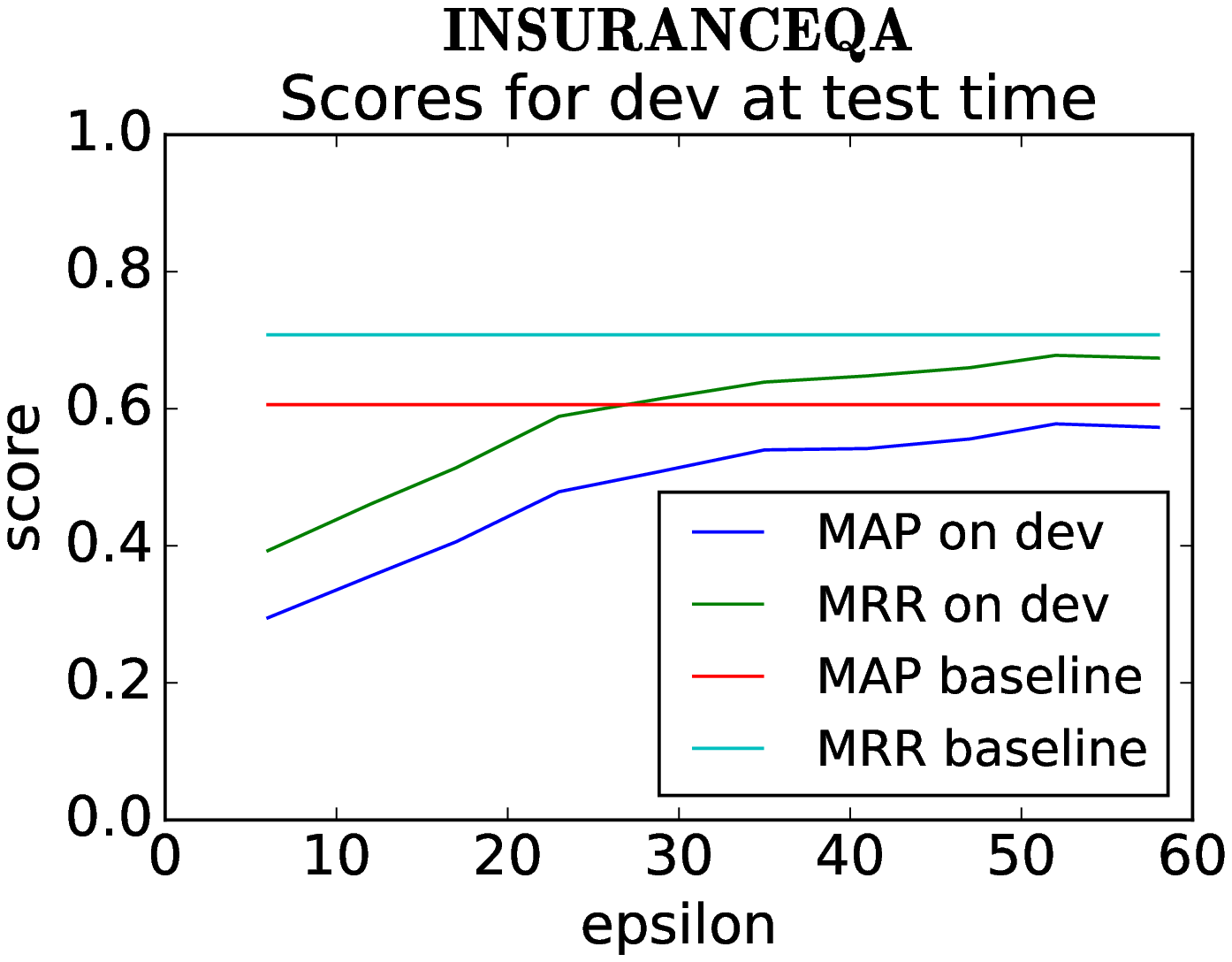}
\end{subfigure}
\caption{$d_{\chi}$-privacy scores against utility baseline}% for different values of $\varepsilon$}
\label{fig:all_results}
\end{figure*}

We now discuss the individual results from running the $d_{\chi}$ algorithm on machine learning models trained on the $3$ datasets presented in \Cref{fig:all_results}. We start with the binary sentiment classification task on the IMDb dataset. 
Across the $3$ experiments, we observe the expected privacy utility trade-off. As $\varepsilon$ increases (greater privacy loss), the utility scores improve. Conversely, at smaller values of $\varepsilon$, we record worse off scores. However, this observation varies across the tasks. For the binary classification task, at training and test time, the model remains robust to the injected perturbations. Performance degrades on the other $2$ tasks with the question answering task being the most sensitive to the presence of noise.

\section{\ac{ML} Privacy Experiments}\label{sec:ml_priv_experiments}
We now describe how we evaluate the privacy guarantees from our approach against two query scrambling methods from literature.

\subsection{Baselines for privacy experiments}
We evaluated our approach against the following baselines:

\textbf{Versatile} \cite{arampatzis2015versatile} -- using the `semantic' and `statistical' query scrambling techniques. Sample queries were obtained from the paper. 

\textbf{Incognito} \cite{masood2018incognito} -- using perturbations `with' and `without' noise. Sample queries were also obtained from the paper. 

\subsection{Datasets}
\textbf{Search logs} -- The two evaluation baselines \cite{arampatzis2015versatile,masood2018incognito} sampled data from \cite{pass2006picture} therefore, we also use this as the dataset for our approach.

\subsection{Setup for privacy experiments}
We evaluate the baselines and our approach using the privacy auditor described in \cite{songshmatikovkdd}. We modeled our experiments after the paper as follows: From the search logs dataset \cite{pass2006picture}, we sampled users with between $150$ and $500$ queries resulting in $8,670$ users. We randomly sampled $100$ users to train and another $100$ users (negative examples) to test the privacy auditor system.

For each evaluation baseline, we dropped an existing user, then created a new user and injected the scrambled queries using the baseline's technique. The evaluation metrics are: Precision, Recall, Accuracy and \ac{AUC}. The metrics are computed over the ability of the privacy auditor to correctly identify queries \emph{used} to train the system, and queries \emph{not used} to train the system.

\subsection{Results for privacy experiments}
%The results in \Cref{tab:priv_baselines} effectively demonstrate that the existing baselines fail to prevent attacks by the privacy auditor. The auditor is able to perfectly identify queries that were perturbed using the baseline techniques regardless of whether they were actually used to train the system or not (\ie{} as true positives or true negatives).

The metrics in \Cref{tab:priv_baselines} depict privacy loss (\ie{} lower is better). The results highlight that existing baselines fail to prevent attacks by the privacy auditor. The auditor is able to perfectly identify queries that were perturbed using the baseline techniques regardless of whether they were actually used to train the system or not.

\begin{table}[h]
\smaller
\begin{center}
\begin{tabular}{ l | c | c | c | c  }
 \toprule
\textbf{Model} & \textbf{Precision} & \textbf{Recall} & \textbf{Accuracy} & \textbf{AUC} \\ 
 \hline
Original queries & 1.0 & 1.0 & 1.0 & 1.0 \\
\hline
Versatile (semantic) & 1.0 & 1.0 & 1.0 & 1.0 \\
Versatile (statistical) & 1.0 & 1.0 & 1.0 & 1.0 \\
\hline
Incognito (without noise) & 1.0 & 1.0 & 1.0 & 1.0 \\
Incognito (with noise) & 1.0 & 1.0 & 1.0 & 1.0 \\
\hline
$d_{\chi}$-privacy (at $\varepsilon = 23$) & \textbf{0.0} & \textbf{0.0} & \textbf{0.5} & \textbf{0.36} \\
\bottomrule
\end{tabular}
\caption{Results: scores measure privacy loss (lower is better)}
\label{tab:priv_baselines}
\end{center}
\end{table}

%precision=0.0, recall=0.0, acc=0.5, auc=0.06380000000000001 for prop=1.0 queries
%precision=0.0, recall=0.0, acc=0.5, auc=0.04410000000000002 for prop=1.0 queries
%precision=0.0, recall=0.0, acc=0.5, auc=0.1075 for prop=1.0 queries
%precision=0.0, recall=0.0, acc=0.5, auc=0.36149999999999993 for prop=1.0 queries
%precision=0.6666666666666666, recall=0.02, acc=0.505, auc=0.6128 for prop=1.0 queries
%precision=1.0, recall=0.09, acc=0.545, auc=0.8515 for prop=1.0 queries
%precision=0.9333333333333333, recall=0.14, acc=0.565, auc=0.8756 for prop=1.0 queries
%precision=1.0, recall=0.3, acc=0.65, auc=0.9307 for prop=1.0 queries
%precision=1.0, recall=0.5, acc=0.75, auc=0.9822000000000001 for prop=1.0 queries
%precision=1.0, recall=0.62, acc=0.81, auc=0.9897 for prop=1.0 queries

\noindent Conversely, our $d_{\chi}$ approach in the last line of \Cref{tab:priv_baselines} and expanded in \Cref{tab:priv_dx}, show we are able provide tunable privacy guarantees (over $3$x greater than baselines for $\varepsilon \leq 23$ on \ac{AUC} scores). Across all metrics (at $\varepsilon \leq 23$), our privacy guarantees is better than chance.

\begin{table}[h]
\smaller
\begin{center}
\begin{tabular}{ l | c | c | c | c | c | c | c | c | c    }
 \hline
  & \multicolumn{9}{c}{$\varepsilon$ for \textsc{GloVe} $300d$} \\
 \hline
\textbf{Metric} & $\mathbf{\gloveScale{100}}$ & $\mathbf{\gloveScale{200}}$ & $\mathbf{\gloveScale{300}}$ & $\mathbf{\gloveScale{400}}$ & $\mathbf{\gloveScale{500}}$ & $\mathbf{\gloveScale{600}}$ & $\mathbf{\gloveScale{700}}$ & $\mathbf{\gloveScale{800}}$ & $\mathbf{\gloveScale{900}}$ \\ 
 \hline
Precision & 0.00 & 0.00 & 0.00 & 0.00 & 0.67 & 0.90 & 0.93 & 1.00 & 1.00 \\
Recall & 0.00 & 0.00 & 0.00 & 0.00 & 0.02 & 0.09 & 0.14 & 0.30 & 0.50 \\
Accuracy & 0.50 & 0.50 & 0.50 & 0.50 & 0.51 & 0.55 & 0.57 & 0.65 & 0.75 \\
\ac{AUC} & 0.06 & 0.04 & 0.11 & 0.36 & 0.61 & 0.85 & 0.88 & 0.93 & 0.98 \\
\hline
\end{tabular}
\caption{$d_{\chi}$-privacy results: all scores measure privacy loss}
\label{tab:priv_dx}
\end{center}
\end{table}

\section{Discussion}
We have described how to achieve formal privacy guarantees in textual datasets by perturbing the words of a given query. Our experiments on machine learning models using different datasets across different task types have provided empirical evidence into the feasibility of adopting this technique. 
%\subsection{Privacy--Utility Trade-off}
Overall, our findings demonstrate the tradeoffs between desired privacy guarantees and the achieved task utility. Previous work in data mining \cite{brickell2008cost,li2009tradeoff} and privacy research \cite{geng2014optimal,he2014blowfish} have described the cost of privacy and the need to attain tunable utility results. Achieving optimal privacy as described in \emph{Dalenius's Desideratum} \cite{dwork2011firm}
%, such that `anything that can be learned about a respondent from a statistical database should be learnable without access to the database' 
will yield a dataset that confers no utility to the curator. While techniques such as homomorphic encryption \cite{gentry2009fully} hold promise, they have not been developed to the point of practical applicability.

\section{Related work}

Text redaction for privacy protection is a well understood and widely studied problem \cite{butler2004us} with good solutions currently found wanting \cite{hill2016effectiveness}. This is amplified by the fact that redaction needs vary. For example, with transactional data (such as search logs), the objective is \emph{anonymity} or \emph{plausible deniability} so that the identity of the person performing the search cannot be ascertained. On the other hand, with plain text data (such as emails and medical records), the objective might be \emph{confidentiality} so that an individual is not associated with an entity. Our research is designed around conferring plausible deniability in search logs while creating a mechanism that can be extended to other plain text data types.

The general approach to text redaction in literature follows two steps: (i) detection of sensitive terms, and (ii) obfuscation of the identified entities. Our approach differs from related works in these two tasks. With respect to (i) in transactional data, \cite{masood2018incognito} is predicated on defining private queries and sensitive terms based on uniformity, uniqueness and linkability to predefined \ac{PII} such as names and locations. This approach however doesn't provide privacy guarantees for queries that fall outside this definition. Other methods such as  \cite{domingo2009h,pang2010embellishing,sanchez2013knowledge} bypass the detection of sensitive terms and inject additional keywords into the initial search query. It has been shown in  \cite{petit2015peas} that this model is susceptible to de-anonymisation attacks. On the other hand, techniques such as \cite{arampatzis2015versatile} do not focus on (i), but source (ii) i.e., replacement entities from related web documents while we use word embedding models for this step.

Similarly, for plain text data, approaches such as \cite{cumby2010inference,cumby2011machine} address (i) by using models to `recognize several classes of PII' such as names and credit cards, while \cite{sanchez2016c} focuses on (ii) that is, sanitizing an entity \emph{c} by removing all terms \emph{t} that can identify \emph{c} individually or in aggregate in a knowledge base \emph{K}. Indeed, any privacy preserving algorithm that places \apriori{} classification on sensitive data types assume boundaries on an attackers side knowledge and a finite limit on potentially new classes of personal identifiers. Our approach with $d_{\chi}$-privacy aims to do away with such assumptions to provide tunable privacy guarantees.

\section{Conclusion}

In this paper, we presented a formal approach to carrying out privacy preserving text perturbation using $d_{\chi}$-privacy. Our approach applied carefully calibrated noise to vector representations of words in a high dimension space as defined by word embedding models. We presented a theoretical privacy proof that satisfies $d_{\chi}$-privacy where the parameter $\varepsilon$ provides guarantees with respect to a metric $d(x,x')$ defined by the word embedding space. Our experiments demonstrated that our approach provides tunable privacy guarantees over $3$ times greater than the baselines, while incurring $< 2\%$ utility loss on training binary classifiers (among other task types) for a range of $\varepsilon$ values. By combining the results of our privacy and utility experiments, with our guidelines on selecting $\varepsilon$ by using worst-case guarantees from our plausible deniability statistics, data holders can make a rational choice in applying our mechanism to attain a suitable privacy-utility tradeoff for their tasks.

%
%\begin{acks}
%  The authors would like to thank ...
%
%\end{acks}